\newif\ifsicomp
\sicompfalse
\ifsicomp
\documentclass[final,leqno,onefignum,onetabnum]{siamltex1213}
\usepackage{amsfonts}%boxedminipage,color,url,fullpage}
\else
\documentclass[11pt,twoside]{article}
\usepackage{fullpage,amsthm}
\usepackage[colorlinks=true,linkcolor=DarkBlue,citecolor=DarkGreen]{hyperref}
\fi
\usepackage{amsmath,amssymb,xspace}
\usepackage{graphicx}
\usepackage{color}
\usepackage{algorithmic}
\usepackage{algorithm}
\usepackage{microtype}
\usepackage[small]{caption}
\definecolor{DarkGreen}{rgb}{0.1,0.5,0.1}
\definecolor{DarkRed}{rgb}{0.5,0.1,0.1}
\definecolor{DarkBlue}{rgb}{0.1,0.1,0.5}
\definecolor{mygreen}{rgb}{0,0.6,0}
\definecolor{mygray}{rgb}{0.5,0.5,0.5}
\definecolor{mymauve}{rgb}{0.58,0,0.82}

\def\draft{1}
\ifnum\draft=1 % show authors' note if draft
    \def\ShowAuthNotes{1}
\else
    \def\ShowAuthNotes{0}
\fi

\newcommand{\cS}{\mathcal{S}}

\newcommand{\eps}{\varepsilon}

\newcommand{\R}{\mathbb{R}}

\newcommand{\A}{{\cal A}}

\newcommand{\eg}{{\it e.g.}}

\newcommand{\calP}{\mathcal {P}}

\newcommand{\dba}{x}
\newcommand{\dbb}{y}

\newcommand{\X}{\mathcal{X}}
\newcommand{\uni}{\mathcal{X}}
\newcommand{\rowdom}{\mathcal{X}^n}

\newcommand{\set}[1]{\left\{#1\right\}}

\newcommand{\poly}{{\sf poly}}

            \newcommand{\cm}{\mathcal{M}}     \newcommand{\cs}{\mathcal{S}}

\newcommand{\pr}{\Pr} % Generic probability

\newcommand{\remove}[1]{}
\newcommand{\op}[1]{}

\ifsicomp
\else
\newtheorem{theorem}{Theorem}

\newtheorem{corollary}[theorem]{Corollary}

\newtheorem{lemma}[theorem]{Lemma}

\newtheorem{proposition}[theorem]{Proposition}

\newtheorem{definition}[theorem]{Definition}

\fi

\newtheorem{fact}[theorem]{Fact}
\newtheorem{remark}[theorem]{Remark}

\ifnum\ShowAuthNotes=1
\newcommand{\authnote}[2]{{ \footnotesize \bf{\color{DarkRed}[#1's Note:
{\color{DarkBlue}#2}]}}}
\else
\newcommand{\authnote}[2]{}
\fi

\definecolor{darkred}{rgb}{0.5,0,0}
\definecolor{brown}{rgb}{0.5,0.25,0}
\definecolor{darkgreen}{rgb}{0,.5,0}
\definecolor{lightgray}{gray}{.8}
\definecolor{lightrose}{rgb}{1, .7, .7}
\definecolor{lightskyblue}{rgb}{.9,.9,1}
\definecolor{aquamarine}{rgb}{.442,1,.812}
\definecolor{lightgreen}{rgb}{.7,1,.7}
\definecolor{lightpink}{rgb}{1,.8,0.8}
\definecolor{thistle}{rgb}{0.9,0.749,0.9}
\definecolor{lightblue}{rgb}{0.67843,0.847,0.9}

\usepackage{vitalymacros}

\usepackage{boxedminipage}
\usepackage{bm}

\renewcommand{\exp}{\mathrm{exp}}
\renewcommand{\dba}{S}
\renewcommand{\dbb}{S'}
\renewcommand{\cs}{{\cal O}}
\renewcommand{\cm}{{\cal A}}
\newcommand{\Esymb}{\mathbb{E}}
\newcommand{\Psymb}{\mathbb{P}}
\DeclareMathOperator*{\Expect}{\Esymb}
\DeclareMathOperator*{\Prob}{\Psymb}
\renewcommand{\E}{\Expect}
\renewcommand{\pr}{\Prob}
\renewcommand{\Pr}{\Prob}

\newcommand{\cE}{{\cal E}}

\usepackage{listings}

\newcommand\bi{I}
\newcommand{\bY}{{\bm Y}}
\newcommand{\bS}{{\bm S}}
\newcommand{\bT}{{\bm T}}
\newcommand{\bphi}{{\bm \phi}}

\begin{document}
\title{Preserving Statistical Validity in Adaptive Data Analysis\thanks{Preliminary version of this work appears in the proceedings of the ACM Symposium on Theory of Computing (STOC), 2015}}
\ifsicomp
\author{Cynthia Dwork\thanks{Microsoft Research (\email{dwork@microsoft.com}).} \and Vitaly Feldman\thanks{IBM Almaden Research Center (\email{vitaly@post.harvard.edu}). Part of this work done while visiting the Simons Institute, UC Berkeley.}  \and Moritz Hardt\thanks{IBM Almaden Research Center (\email{m@mrtz.org}).} \and Toniann Pitassi\thanks{University of Toronto (\email{toni@cs.toronto.edu}).} \and Omer Reingold\thanks{Samsung Research America (\email{omer.reinngold@gmail.com}).} \and Aaron Roth\thanks{Department of Computer and Information Science, University of Pennsylvania (\email{aaroth@cis.upenn.edu}).}}
\else
\author{Cynthia Dwork\thanks{Microsoft Research} \and Vitaly Feldman\thanks{IBM Almaden Research Center. Part of this work done while visiting the Simons Institute, UC Berkeley}  \and Moritz Hardt\thanks{IBM Almaden Research Center} \and Toniann Pitassi\thanks{University of Toronto} \and Omer Reingold\thanks{Samsung Research America} \and Aaron Roth\thanks{Department of Computer and Information Science, University of Pennsylvania}}
\fi

\date{}
\maketitle

\begin{abstract}
A great deal of effort has been devoted to reducing the risk of spurious
scientific discoveries, from the use of sophisticated validation techniques,
to deep statistical methods for controlling the false discovery rate in
multiple hypothesis testing.  However, there is a fundamental disconnect
between the theoretical results and the practice of data analysis: the theory
of statistical inference assumes a fixed collection of hypotheses to be
tested, or learning algorithms to be applied, selected non-adaptively before
the data are gathered, whereas in practice data is shared and reused with
hypotheses and new analyses being generated on the basis of data exploration
and the outcomes of previous analyses.

In this work we initiate a principled study of how to guarantee the validity
of statistical inference in adaptive data analysis. As an instance of this
problem, we propose and investigate the question of estimating the
expectations of $m$ adaptively chosen functions on an unknown distribution
given $n$ random samples.

We show that, surprisingly, there is a way to estimate an \emph{exponential}
in $n$ number of expectations accurately even if the functions are chosen
adaptively.  This gives an exponential improvement over standard empirical
estimators that are limited to a linear number of estimates.  Our result
follows from a general technique that counter-intuitively involves actively
perturbing and coordinating the estimates, using techniques developed for
privacy preservation.  We give additional applications of this technique to
our question.
\end{abstract}

\ifsicomp
\pagestyle{myheadings}
\thispagestyle{plain}

\else
\vfill
\thispagestyle{empty}

\pagebreak
\fi
\section{Introduction}

Throughout the scientific community there is a growing recognition that
claims of statistical significance in published research are frequently
invalid \cite{Ioannidis05,Ioannidis05b,PrintzSA11,BegleyEllis12}. The past
few decades have seen a great deal of effort to understand and propose
mitigations for this problem. These efforts range from the use of
sophisticated validation techniques and deep statistical methods for
controlling the false discovery rate in multiple hypothesis testing to
proposals for preregistration (that is, defining the entire data-collection
and data-analysis protocol ahead of time). The statistical inference theory
surrounding this body of work assumes a fixed procedure to be performed,
selected before the data are gathered. In contrast, the practice of data
analysis in scientific research is by its nature an adaptive process, in
which new hypotheses are generated and new analyses are performed on the
basis of data exploration and observed outcomes on the same data. This
disconnect is only exacerbated in an era of increased amounts of open access
data, in which multiple, mutually dependent, studies are based on the same
datasets.

It is now well understood that adapting the analysis to data ({\it e.g.},
choosing what variables to follow, which comparisons to make, which tests to
report, and which statistical methods to use) is an implicit multiple
comparisons problem that is not captured in the reported significance levels
of standard statistical procedures. This problem, in some contexts referred
to as ``p-hacking'' or ``researcher degrees of freedom'', is one of the primary
explanations of why research findings are frequently false
\cite{Ioannidis05,SimmonsNS11,GelmanLoken13}.
\remove{ % We could keep Freedman's paradox but in the interest of brevity perhaps not in 10-page version
 Although not usually
understood in these terms, ``Freedman's paradox'' is an elegant demonstration
of the powerful effect of adaptive analysis on significance levels
\cite{Freedman83}. In Freedman's simulation an equation is fitted, variables
with an insignificant $t$-statistic are dropped and the equation is re-fit to
this new---adaptively chosen---subset of variables, with famously misleading
results:  when the relationship between the dependent and explanatory
variables is non-existent, the procedure overfits, erroneously declaring
significant relationships.
}
%We may think of this adaptive procedure as generating a hypothesis for which the actual dataset is not representative of the distribution from which the dataset was drawn.

The ``textbook'' advice for avoiding problems of this type is to collect fresh
samples from the same data distribution whenever one ends up with a procedure
that depends on the existing data. Getting fresh data is usually costly and
often impractical so this requires partitioning the available dataset
randomly into two or more disjoint sets of data (such as a training and
testing set) prior to the analysis. Following this approach conservatively
with $m$ adaptively chosen procedures would significantly (on average by a
factor of $m$) reduce the amount of data available for each procedure. This
would be prohibitive in many applications, and as a result, in practice even
data allocated for the sole purpose of testing is frequently reused (for
example to tune parameters). Such abuse of the holdout set is well known to result
in significant  overfitting to the holdout or cross-validation set \cite{Reunanen:03,RaoF:08}.

%\vnote{This example could be omitted from this version but seems potentially useful in making the issue a bit more concrete } \anote{Lets keep this paragraph, it makes the point we are making seem more grounded in reality.}

Clear evidence that such reuse leads to overfitting can be seen in the data
analysis competitions organized by Kaggle Inc. In these competitions, the
participants are given training data and can submit (multiple) predictive
models in the course of competition. Each submitted model is evaluated on a
(fixed) test set that is available only to the organizers. The score of each
solution is provided back to each participant, {\em who can then submit a new
model}. In addition the scores are published on a public leaderboard. At the
conclusion of the competition the best entries of each participant are
evaluated on an additional, hitherto unused, test set. The scores from these
final evaluations are published. The comparison of the scores on the
adaptively reused test set and one-time use test set frequently reveals
significant overfitting to the reused test set
(e.g.~\cite{KaggleBlogWind,KaggleRolie}), a well-recognized issue frequently
discussed on Kaggle's blog and user forums \cite{KaggleBlog,KaggleForums}.

Despite the basic role that adaptivity plays in data analysis we are not
aware of previous general efforts to address its effects on the statistical
validity of the results (see Section \ref{sec:related} for an overview of
existing approaches to the problem). We show that, surprisingly, the
challenges of adaptivity can be addressed using insights from {\em
differential privacy}, a definition of privacy tailored to privacy-preserving
data analysis. Roughly speaking, differential privacy ensures that the
probability of observing any outcome from an analysis is ``essentially
unchanged'' by modifying any single dataset element (the probability
distribution is over the randomness introduced by the algorithm).
Differentially private algorithms permit a data analyst to learn about the
dataset as a whole (and, by extension, the distribution from which the data
were drawn), while simultaneously protecting the privacy of the individual
data elements.  Strong composition properties show this holds even when the
analysis proceeds in a sequence of adaptively chosen, individually
differentially private, steps.

\subsection{Problem Definition}
We consider the standard setting in statistics and statistical learning theory: an analyst is given samples drawn randomly and independently from some unknown distribution $\calP$ over a discrete universe $\uni$ of possible data points. % The goal of the analyst is to produce hypotheses or models that {\em generalize} to the underlying distribution rather than to learn about the quirks of a particular data set.
While our approach can be applied to any output of data analysis, we focus on
real-valued functions defined on $\uni$. Specifically, for a function
$\psi\colon {\cal X}\to [0,1]$ produced by the analyst we consider the task of
estimating the expectation $\calP[\psi]=\E_{x\sim \calP}[\psi(x)]$ up to some
additive error~$\tau$ usually referred to as \emph{tolerance}. We require the
estimate to be within this error margin with high probability.
%(or, equivalently, a confidence
%interval with high confidence level).
%Note that this is different from the
%{\em empirical} average on a dataset $S=(x_1,\dots,x_n)$ that we denote by
%$\cE_S \frac1n\sum_{i=1}^n\phi(x_i)$.

We choose this setup for three reasons.
First, a variety of quantities of interest in data analysis can be expressed in this form for some function $\psi$. For example, true means and moments of individual attributes, correlations between attributes and the generalization error of a predictive model or classifier.
Next, a request for such an estimate is referred to as a \emph{statistical query} in the context of the well-studied statistical query model~\cite{Kearns93,FeldmanGRVX:13}, and it is known that using statistical queries in place of direct access to data it is possible to implement most standard analyses used on i.i.d.\,data (see \cite{Kearns93,BlumDMN05,ChuKLYBNO:06} for examples).
Finally, the problem of providing accurate answers to a large number of queries for the average
value of a function on the dataset has been the subject of intense
investigation in the differential privacy literature.\footnote{The average
value of a function $\psi$ on a set of random samples is a natural estimator
of $\calP[\psi]$. In the differential privacy literature such queries are
referred to as {\em (fractional) counting queries}.}

We address the following basic question: how many adaptively chosen
statistical queries can be correctly answered using $n$ samples drawn
i.i.d.~from $\cP$? The conservative approach of using fresh samples for each
adaptively chosen query would lead to a sample complexity that scales linearly
with the number of queries $m$. We observe that such a bad dependence is
inherent in the standard approach of estimating expectations by the exact
empirical average on the samples. This is directly implied by the techniques
from~\cite{DN03} who show how to make linearly many non-adaptive counting
queries to a dataset, and reconstruct nearly all of it. Once the data set is
nearly reconstructed it is easy to make a query for which the empirical
average on the dataset is far from the true expectation. Note that this
requires only a single round of adaptivity! A simpler and more natural
example of the same phenomenon is known as ``Freedman's paradox''
\cite{Freedman83} and we give an additional simple example in the Appendix.
This situation is in stark contrast to the non-adaptive case in which $n =
O\left(\frac{\log m}{\tau^2}\right)$ samples suffice to answer $m$ queries
with tolerance $\tau$ using empirical averages. Below we refer to using
empirical averages to evaluate the expectations of query functions as the
\emph{na\"ive} method.

\subsection{Our Results}
\label{sec:our-results}

Our main result is a broad \emph{transfer theorem} showing that any adaptive
analysis that is carried out in a differentially private manner must lead to a
conclusion that generalizes to the underlying distribution. This theorem
allows us to draw on a rich body of results in differential privacy and to
obtain corresponding results for our problem of guaranteeing validity in
adaptive data analysis. Before we state this general theorem, we describe a
number of important corollaries for the question we formulated above.

Our primary application is that, remarkably, it is possible to
answer nearly \emph{exponentially many} adaptively chosen statistical queries (in the size of the data set $n$).
Equivalently, this reduces the \emph{sample complexity} of answering $m$ queries from \emph{linear} in the number of queries to \emph{polylogarithmic}, nearly matching the dependence that is necessary for non-adaptively chosen queries.
\begin{theorem}[Informal]
\label{thm:mwu-intro}
There exists an algorithm that given a dataset of size at least $n \geq \min(n_0,n_1)$, can answer any $m$ adaptively chosen statistical queries so that with high probability, each answer is correct up to tolerance $\tau$, where:
\[n_0 = O\left(\frac{(\log m)^{3/2} \sqrt{\log |\uni|}}{\tau^{7/2}}\right) \ \ \ \textrm{and}\ \ \ n_1 = O\left(\frac{\log m \cdot \log |\uni|}{\tau^{4}}\right) . \]
\end{theorem}
The two bounds above are incomparable. Note that the first bound is larger than the sample complexity needed to answer non-adaptively chosen queries by only a factor of $O\left(\sqrt{\log m\log |\uni|}/\tau^{3/2}\right)$, whereas the second one is larger by a factor of $O\left(\log (|\uni|)/\tau^2\right)$. Here $\log |\uni|$ should be viewed as roughly the \emph{dimension} of the domain. For example, if the underlying domain is $\uni = \{0,1\}^d$, the set of all possible vectors of $d$-boolean attributes, then $\log|\uni| = d$.

The above mechanism is not computationally efficient (it has running time
linear in the size of the data universe $|\uni|$, which is \emph{exponential}
in the dimension of the data). A natural question raised by our result is
whether there is an efficient algorithm for the task. This question was
addressed in \cite{HU14,SU14} who show that under standard cryptographic
assumptions any algorithm that can answer more than $\approx n^2$ adaptively
chosen statistical queries must have running time exponential in~$\log |\uni|$.

We show that it is possible to match this quadratic lower bound using a simple
and practical algorithm that perturbs the answer to each query with
independent noise.
%\Mnote{Not clear what ``main result'' refers to:}
%The following theorem, which we prove as a corollary of our
%main result, achieves essentially the best possible bound among all
%computationally efficient mechanisms, since it matches the lower bounds proven
%by \cite{HU14,SU14}.

\begin{theorem}[Informal]
\label{thm:laplace-intro}
There exists a computationally efficient algorithm for answering $m$ adaptively chosen statistical queries, such that with high probability, the answers are correct up to tolerance $\tau$, given a data set of size at least $n \geq n_0$ for:
$$n_0 = O\left(\frac{\sqrt{m}(\log m)^{3/2}}{\tau^{5/2}}\right)\ .$$
\end{theorem}

Finally, we show a computationally efficient method which can answer
\emph{exponentially many} queries so long as they were generated using $o(n)$
rounds of adaptivity, even if we do not know where the rounds of adaptivity
lie. Another practical advantage of this algorithm is that it only pays the
price for a round if adaptivity actually causes overfitting. In other words,
the algorithm does not pay for the adaptivity itself but only for the actual
harm to statistical validity that adaptivity causes. This means that in many
situations it would be possible to use this algorithm successfully with a much
smaller ``effective" $r$ (provided that a good bound on it is known).
\begin{theorem}[Informal]
\label{thm:reusable-intro}
There exists a computationally efficient algorithm for answering $m$ adaptively chosen statistical queries, generated in $r$ rounds of adaptivity, such that with high probability, the answers are correct up to some tolerance $\tau$, given a data set of size at least $n \geq n_0$ for:
$$n_0 = O\left(\frac{r \log m}{\tau^2}\right)\ .$$
\end{theorem}
Formal statements of these results appear in Section \ref{sec:applications}.

\subsection{Overview of Techniques}
We consider a setting in which an \emph{arbitrary} adaptive data analyst
chooses queries to ask (as a function of past answers), and receives answers
from an algorithm referred to as an \emph{oracle} whose input is a
dataset $\bS$ of size $n$ randomly drawn from $\cP^n$. To begin with, the oracles we
use will only guarantee accuracy with respect to the empirical average on
their input dataset $\bS$, but they will simultaneously guarantee
differential privacy. We exploit a crucial property about differential
privacy, known as its post-processing guarantee: \emph{Any} algorithm that
can be described as the (possibly randomized) post-processing of the output
of a differentially private algorithm is itself differentially private.
Hence, although we do not know how an arbitrary analyst is adaptively
generating her queries, we do know that if the only access she has to $\bS$
is through a differentially private algorithm, then her method of producing
query functions must be differentially private with respect to $\bS$. We can
therefore, without loss of generality, think of the oracle and the analyst as a
\emph{single} algorithm $\cm$ that is given a random data set $\bS$ and
returns a differentially private output query $\bphi=\cm(\bS).$ Note that
$\bphi$ is random both due to the internal randomness of $\cm$ and the
randomness of  the data~$\bS.$ This picture is the starting point of our
analysis, and allows us to study the generalization properties of queries
which are \emph{generated} by differentially private algorithms, rather than
estimates returned by them.

Our results then follow from a strong connection we make between
\emph{differential privacy} and \emph{generalization}, which will likely have
applications beyond those that we explore in this paper. At a high level, we
prove that if $\cm$ is a differentially private algorithm then the empirical
average of a function that it outputs on a random dataset will be close to
the true expectation of the function with high probability\footnote{A weaker connection that gives closeness {\em in expectation} over the dataset and algorithm's randomness was known to some experts and is considered folklore. We give a more detailed comparison in Sec.~\ref{sec:related} and Sec.~\ref{sec:basic-conn}.} (over the choice
of the dataset and the randomness of $\cm$). More formally, for a dataset
$S=(x_1,\dots,x_n)$ and a function $\psi:\uni \rar [0,1]$, let $\cE_S[\psi] =
\frac1n\sum_{i=1}^n\psi(x_i)$ denote the empirical average of $\psi$. We
denote a random dataset chosen from $\cP^n$ by $\bS$. For any fixed function
$\psi$, the empirical average $\cE_{\bS}[\psi]$ is strongly concentrated around its
expectation $\cP[\psi]$. However, this statement is no longer true if $\psi$
is allowed to depend on $\bS$ (which is what happens if we choose functions
adaptively, using previous estimates on $\bS$). However for a hypothesis
output by a differentially private $\cm$ on $\bS$ (denoted by
$\bphi=\cm(\bS)$), we show that $\cE_{\bS}[\bphi]$ is close to $\cP[\bphi]$
with high probability.

High probability bounds are necessary to ensure that valid answers can be
given to an exponentially large number of queries. To prove these bounds, we
show that differential privacy roughly preserves the moments of
$\cE_{\bS}[\bphi]$ even when conditioned on $\bphi = \psi$ for any fixed
$\psi$. Now using strong concentration of the $k$-th moment of
$\cE_{\bS}[\psi]$ around $\cP[\psi]^k$, we can obtain that $\cE_{\bS}[\bphi]$
is concentrated around $\cP[\bphi]$. Such an argument works only for
$(\eps,0)$-differential privacy due to conditioning on the event $\bphi =
\psi$ which might have arbitrarily low probability. We use a more delicate
conditioning to obtain the extension to $(\eps,\delta)$-differential privacy.
We note that $(\eps,\delta)$-differential privacy is necessary to obtain the
stronger bounds that we use for Theorems~\ref{thm:mwu-intro} and~\ref{thm:laplace-intro}.

We give an alternative, simpler proof for $(\eps,0)$-differential privacy
that, in addition, extends this connection beyond expectations of functions.
We consider a differentially private algorithm $\cm$ that maps a database $\bS
\sim \cP^n$ to elements from some arbitrary range $Z$. Our proof shows that if we have a collection of events $R(y)$ defined over databases, one for each element $y \in Z$, and each event is individually unlikely in the sense that for all $y$, the probability that $\bS \in R(y)$ is small, then the probability remains small that $\bS \in R(\bY)$, where $\bY = \cm(\bS)$. Note that this statement involves a re-ordering of quantifiers. The hypothesis of the theorem says that the probability of event $R(y)$ is small for each $y$, where the randomness is taken over the choice of database $\bS \sim \cP^n$, which is independent of $y$. The conclusion says that the probability of $R(\bY)$ remains small, even though $\bY$ is chosen as a function of $\bS$, and so is no longer independent. The upshot of this result is that \emph{adaptive} analyses, if performed via a differentially private algorithm, can be thought of (almost) as if they were non-adaptive, with the data being drawn \emph{after} all of the decisions in the analysis are fixed.

To prove this result we note that it would suffice to establish that for every $y \in Z$, $\pr[\bS \in R(y) \cond \bY = y]$ is not much larger than $\pr[\bS \in R(y)]$. By Bayes' rule, for every dataset $S$,
\[ \frac{\pr[\bS = S \cond \bY = y]}{\pr[\bS = S]} = \frac{\pr[\bY = y \cond \bS =S]}{\pr[\bY =y]} .\]
Therefore, to bound the ratio of $\pr[\bS \in R(y) \cond \bY = y]$ to $\pr[\bS \in R(y)]$ it is sufficient to bound the ratio of $\pr[\bY = y \cond \bS =S]$ to $\pr[\bY =y]$ for {\em most} $S \in R(y)$. Differential privacy implies that $\pr[\bY = y \cond \bS =S]$ does not change fast as a function of $S$. From here, using McDiarmid's concentration inequality, we obtain that $\pr[\bY = y \cond \bS =S]$ is strongly concentrated around its mean, which is exactly $\pr[\bY =y]$.

%\vnote{Need to add something about the proof here.}

\remove{
% We already mention this before. It seems that if would suffice to have the details in the appendix}
This bound is tight for non-adaptively chosen hypotheses, and is the bound we wish to compete with.  Unfortunately, it is folklore that for $m$ \emph{adaptively} chosen hypotheses, the sample complexity of the naive mechanism must grow linearly with $m$.

 \begin{proposition}[Informal]
 For a data set of size $n$ with elements drawn from the uniform distribution over a universe $\uni$, there exists a set of $O(n\log |\uni|)$ hypotheses such that the naive mechanism does not provide answers that are valid up to error $1/2$.

 Note that this implies that the naive mechanism performs exponentially worse in the face of adaptively chosen hypotheses -- to provide validity for $m$ adaptively chosen hypotheses, it requires $n = \Omega\left(m\right)$, even for constant tolerance $\tau$.
 \end{proposition}

This might seem discouraging, but this is only a lower bound only for the naive method of evaluating hypotheses, and does not preclude the existence of better methods. Indeed, as we show, better methods exist. As a corollary of our main theorem, we leverage existing differentially private algorithms for accurately answering large adaptively chosen sets of queries to give a mechanism that enjoys a much better bound.
}
\subsection{Related Work}
\label{sec:related}
Numerous techniques have been developed by statisticians to address common
special cases of adaptive data analysis. Most of them address a single round
of adaptivity such as variable selection followed by regression on selected
variables or model selection followed by testing and are optimized for
specific inference procedures (the literature is too vast to adequately cover
here, see Ch.~7 in \cite{hastie2009elements} for a textbook introduction and \cite{TaylorT15} for a survey of some recent work). In
contrast, our framework addresses multiple stages of adaptive decisions,
possible lack of a predetermined analysis protocol and is not restricted to any specific procedures.

The traditional perspective on why adaptivity in data analysis invalidates the significance levels of statistical procedures given for the non-adaptive case is that one ends up disregarding all the other possible procedures or tests that would have been performed had the data been different (see \eg~\cite{SimmonsNS11}). It is well-known that when performing multiple tests on the same data one cannot use significance levels of individual tests and instead it is necessary to control
%\Mnote{was: {\em family-wise} -- false discovery is meant to replace
%family-wise bounds}
measures such as the false discovery rate \cite{benjaminihochberg95}. This view makes it necessary to explicitly account for all the possible ways to perform the analysis in order to provide validity guarantees for the adaptive analysis. While this approach might be possible in simpler studies, it is technically challenging and often impractical in more complicated analyses \cite{GelmanLoken13}.

There are procedures for controlling false discovery in a sequential setting in which tests arrive one-by-one \cite{FosterStine08,AharoniNR11,AharoniRosset13}. However the analysis of such tests crucially depends on tests maintaining their statistical properties despite conditioning on previous outcomes. It is therefore unsuitable for the problem we consider here, in which we place no restrictions on the analyst.

The classical approach in theoretical machine learning to ensure that
empirical estimates generalize to the underlying distribution is based on the
various notions of complexity of the set of functions output by the algorithm,
most notably the VC dimension (see \cite{KV94} or \cite{Shalev-ShwartzBen-David:2014} for a textbook
introduction). If one has a sample of data large enough to guarantee
generalization for all functions in some class of bounded complexity, then it
does not matter whether the data analyst chooses functions in this class
adaptively or non-adaptively. Our goal, in contrast, is to prove
generalization bounds \emph{without} making any assumptions about the class
from which the analyst can choose query functions. In this case the adaptive
setting is very different from the non-adaptive setting.

An important line of work~\cite{BousquettE02,MukherjeeNPR06,PoggioRMN04,ShwartzSSS10}
establishes connections between the \emph{stability} of a learning algorithm
and its ability to generalize. Stability is a measure of how much the output of
a learning algorithm is perturbed by changes to its input. It is known that
certain stability notions are necessary and sufficient for generalization. Unfortunately, the stability notions considered in these prior works are not robust to post-processing, and so the stability of a query answering procedure would not guarantee the stability of the query \emph{generating} procedure used by an arbitrary adaptive analyst. They also do not
compose in the sense that running multiple stable algorithms sequentially and
adaptively may result in a procedure that is not stable. Differential privacy is stronger than these previously studied notions of stability, and in particular enjoys strong post-processing and composition guarantees. %(See Section~\ref{sec:background}).
This provides a calculus for building up complex algorithms that satisfy stability guarantees sufficient to give generalization. Past work has considered the generalization properties of one-shot learning procedures. Our work can in part be
interpreted as showing that differential privacy implies generalization in the
adaptive setting, and beyond the framework of learning.

Differential privacy emerged from a line of work
\cite{DN03,DworkNi04,BlumDMN05}, culminating in the definition given by
\cite{DMNS06}. It defines a stability property of an algorithm developed in
the context of data privacy. There is a very large body of work designing
differentially private algorithms for various data analysis tasks, some of
which we leverage in our applications. Most crucially, it is known how to
accurately answer \emph{exponentially many} adaptively chosen queries on a
fixed dataset while preserving differential privacy \cite{RR10,HardtR10}, which is
what yields the main application in our paper, when combined with our main
theorem. See \cite{Dwork11} for a short survey and \cite{DR14} for a textbook
introduction to differential privacy.

For differentially private algorithms that output a hypothesis it has been known as folklore that differential privacy implies
stability of the hypothesis to replacing (or removing) an element of the input dataset. Such stability is long known to imply
generalization {\em in expectation} (\eg~\cite{ShwartzSSS10}). See Section \ref{sec:basic-conn} for more details.
%\footnote{We are aware of these folklore results via conversations with Kunal Talwar and Frank McSherry, who originally discussed these ideas. Frank McSherry also has a blog post on related ideas: \texttt{http://windowsontheory.org/2014/02/04/differential-privacy-for-measure-concentration/}}
Our technique can be seen as a substantial strengthening of this observation: from expectation to high probability bounds (which is crucial for answering many queries), from pure to approximate differential privacy (which is crucial for our improved efficient algorithms), and beyond the expected error of a hypothesis.

\noindent{\bf Further Developments:}
Our work has attracted substantial interest to the problem of statistical validity in adaptive data analysis and its relationship to differential privacy.
Hardt and Ullman \cite{HU14} and Steinke and Ullman \cite{SU14} have proven complementary \emph{computational} lower bounds for
the problem formulated in this work. They show that, under
standard cryptographic assumptions, the exponential running time of the
algorithm instantiating our main result is unavoidable. Specifically, that the square-root dependence on the number of queries in the sample complexity of our efficient algorithm is nearly optimal, among all computationally efficient mechanisms for answering arbitrary statistical queries.

In \cite{DworkFHPRR15:arxiv} we discuss approaches to the problem of adaptive data analysis more generally. We demonstrate how
differential privacy and description-length-based analyses can be used in this context. In particular, we show that the bounds on $n_1$ obtained in Theorem \ref{thm:mwu-intro} can also be obtained by analyzing the transcript of the median mechanism for query answering \cite{RR10} (even without adding noise). Further, we define a notion of approximate max-information between the dataset and the output of the analysis that ensures generalization with high probability, composes adaptively and unifies (pure) differential privacy and description-length-based analyses. We also demonstrate an application of these techniques to the problem of reusing the holdout (or testing) dataset. An overview of this work and  \cite{DworkFHPRR15:arxiv} intended for a broad scientific audience appears in \cite{DworkFHPRR15:science}.

Blum and Hardt \cite{BlumH15} give an algorithm for reusing the holdout dataset specialized to the problem of maintaining an accurate leaderboard for a machine learning competition (such as those organized by Kaggle Inc.~and discussed earlier). Their generalization analysis is based on the description length of the algorithm's transcript.

Our results for approximate ($\delta > 0$) differential privacy apply only to statistical queries (see Thm.~\ref{thm:simulate-direct-delta}). Bassily, Nissim, Smith, Steinke, Stemmer and Ullman \cite{BassilyNSSSU15} give a novel, elegant analysis of the $\delta > 0$ case that gives an exponential improvement in the dependence on $\delta$ and generalizes it to arbitrary low-sensitivity queries. This leads to stronger bounds on sample complexity that remove an $O(\sqrt{\log (m)/\tau})$ factor from the bounds on $n_0$ we give in Theorems \ref{thm:mwu-intro} and \ref{thm:laplace-intro}. It also implies a similar improvement and generalization to low-sensitivity queries in the reusable holdout application \cite{DworkFHPRR15:arxiv}.

Another implication of our work is that composition and post-processing properties (which are crucial in the adaptive setting) can be ensured by measuring the effect of data analysis on the probability space of the analysis outcomes. Several additional techniques of this type have been recently analyzed. Bassily \etal \cite{BassilyNSSSU15} show that generalization in expectation (as discussed in Cor.~\ref{cor:dp-implies-gen-exp}) can also be obtained from two additional notions of stability: KL-stability and TV-stability that bound the KL-divergence and total variation distance between output distributions on adjacent datasets, respectively. Russo and Zou \cite{RussoZ15} show that generalization in expectation can be derived by bounding the mutual information between
the dataset and the output of analysis. They give applications of their approach to analysis of adaptive feature selection procedures. We note that these techniques do not imply high-probability generalization bounds that we obtain here and in \cite{DworkFHPRR15:arxiv}.

\providecommand{\cO}{{\mathcal O}}
\providecommand{\cS}{{\mathcal S}}
\section{Preliminaries}
Let $\calP$ be a distribution over a discrete universe $\uni$ of possible data points. For a function $\psi\colon \uni\to [0,1]$ let $\calP[\psi] = \E_{x\sim \cP}[\psi(x)]$. Given a dataset $S=(x_1,\dots,x_n)$, a natural estimator of $\calP[\psi]$ is the
empirical average $\frac1n\sum_{i=1}^n\psi(x_i)$. We let $\cE_S$ denote the
empirical distribution that assigns weight $1/n$ to each of the data points
in~$S$ and thus $\cE_S[\psi]$ is equal to the empirical average of $\psi$ on $S$.
\begin{definition}
A {\em statistical query} is defined by a function $\psi:\uni \rar [0,1]$ and tolerance $\tau$. For distribution $\cP$ over $\uni$ a valid response to such a query is any value $v$ such that $|v - \cP(\psi)| \leq \tau$.
\end{definition}
The standard Hoeffding bound implies that for a fixed query function (chosen
independently of the data) the probability over the choice of the dataset that
$\cE_S[\psi]$ has error greater than~$\tau$ is at most $2\cdot\exp(-2\tau^2
n)$. This implies that an exponential in $n$ number of statistical queries can be
evaluated within $\tau$ as long as the hypotheses do not depend on the data.

%\subsection{Differential Privacy}
We now formally define differential privacy. We say that datasets $\dba,\dbb$ are {\em adjacent} if they differ in a single element.
\begin{definition}{\rm \cite{DMNS06,DKMMN06}}
\label{def:dp}
A randomized algorithm $\cm$ with domain $\rowdom$ is $(\eps, \delta)$-differentially private if for all $\cs \subseteq \mathrm{Range}(\cm)$ and for all pairs of adjacent datasets $\dba,\dbb \in \rowdom$:
$$\Pr[\cm(\dba) \in \cs] \leq \exp(\eps)\Pr[\cm(\dbb)\in \cs] + \delta ,$$
where the probability space is over the coin flips of the algorithm $\cm$.
The case when $\delta = 0$ is sometimes referred to as {\em pure} differential privacy, and in this case we may say simply that $\cm$ is $\eps$-differentially private.
\end{definition}
Appendix~\ref{sec:background} contains additional background that we will
need later on.

\remove{
\subsection{Statistical Queries}
We use $\uni$ to denote a fixed discrete domain and $\cP$ a probability distribution over $\uni$.  For a function $\psi$, we denote $\cP(\psi) = \E_{x\sim \cP}[\psi(x)]$.
\begin{definition}
A {\em statistical query} is defined by a function $\psi:\uni \rar [0,1]$ and tolerance $\tau$. For distribution $\cP$ over $\uni$ a valid response to such a query is any value $v$ such that $|v - \cP(\psi)| \leq \tau$. We also say that such response $v$ is {\em statistically valid} up to $\tau$ for distribution $\cP$.
\end{definition}

We distinguish between \emph{statistical queries}, which are defined over a probability distribution, and \emph{counting queries}, which are defined over finite data set $S$ (possibly sampled from a probability distribution). For the same function $\psi$, the corresponding counting query can serve as an empirical estimate of the value of the statistical query defined by $\psi$.

\begin{definition}
A {\em counting query} is defined by a function $\psi:\uni \rar [0,1]$. For a data set $S = (x_1,x_2,\ldots,x_n) \in \uni^{[n]}$ a response to such a query with accuracy (or utility) $\tau$ is any value $v \in [0,1]$ such that $|v - \fr{n} \sum_{i\in [n]} \psi(x_i)| \leq \tau$.
\end{definition}

We denote $S(\psi) \doteq \E_{x \sim S}[\psi(x)]$, where $x \sim S$ denotes a randomly and uniformly chosen element in $S$.
We refer to any algorithm that has oracle access to either statistical queries
or counting queries as query-based algorithm. In the context of differential
privacy literature such algorithms correspond to an unknown mechanism
producing queries over the data set adaptively (and possibly adversarially).
}

\subsection{Review of the known connection between privacy and generalization}
\label{sec:basic-conn}
We now briefly summarize the basic connection between differential privacy and generalization that is considered folklore. % known to some experts \cite{McSherryTalwar:14}.
This connection follows from an observation that differential privacy implies stability to replacing a single sample in a dataset together with known connection between stability and {\em on-average} generalization. We first state the form of stability that is immediately implied by the definition of differential privacy. For simplicity, we state it only for $[0,1]$-valued functions. The extension to any other bounded range is straightforward.
\begin{lemma}
%\label{lem:moment-0}
\label{lem:dp-implies-stability}
Let $\cA$ be an $(\epsilon,\delta)$-differentially private algorithm
ranging over functions from $\X$ to $[0,1]$. For any pair of adjacent datasets $S$ and $S'$ and $x \in \X$:
\begin{equation*}
\E \lb \cA(S)(x) \rb \leq  e^\eps \cdot \E \lb \cA(S')(x) \rb +\delta,
\end{equation*}
and, in particular,
\begin{equation}
\left| \E \lb \cA(S)(x) \rb - \E \lb \cA(S')(x) \rb \right| \leq e^\eps-1 +\delta.
\label{eq:ro-stability}
\end{equation}
\end{lemma}
Algorithms satisfying equation \eqref{eq:ro-stability} are referred to as {\em strongly-uniform-replace-one stable} with rate $(e^\eps-1+\delta)$ by Shalev-Schwartz \etal \cite{ShwartzSSS10}. It is easy to show and is well-known that replace-one stability implies generalization in expectation, referred to as {\em on-average} generalization \cite[Lemma 11]{ShwartzSSS10}. In our case this connection immediately gives the following corollary.
\begin{corollary}
\label{cor:dp-implies-gen-exp}
Let $\cA$ be an $(\epsilon,\delta)$-differentially private algorithm
ranging over functions from $\X$ to $[0,1]$, let $\cP$ be a distribution over $\X$ and let $\bS$ be an independent random variable distributed according to $\cP^n$. Then
\remove{
\begin{equation*}
\E[\cP[\cA(\bS)]] \leq e^\eps \cdot \E[\cE_{\bS}[\cA(\bS)]] +\delta
\end{equation*}
and, in particular,}
\begin{equation*}
|\E[\cE_{\bS}[\cA(\bS)]] - \E[\cP[\cA(\bS)]] |\leq  e^\eps-1 +\delta.
\end{equation*}
\end{corollary}

This corollary was observed in the context of functions expressing the loss of the hypothesis output by a (private) learning algorithm, that is, $\phi(x) = L(h(x),x)$, where $x$ is a sample (possibly including a label), $h$ is a hypothesis function and $L$ is a non-negative loss function. When applied to such a function, Corollary \ref{cor:dp-implies-gen-exp} implies that the expected true loss of a hypothesis output by an $(\eps,\delta)$-differentially private algorithm is at most $e^\eps-1 +\delta$ larger than the expected empirical loss of the output hypothesis, where the expectation is taken over the random dataset and the randomness of the algorithm. A special case of this corollary is stated in a recent work of Bassily \etal \cite{BassilyST14}. More recently, Wang \etal \cite{WangLF15} have similarly used the stability of differentially private learning algorithms to show a general equivalence of differentially private learning and differentially private empirical loss minimization.

A standard way to obtain a high-probability bound from a bound on expectation in Corollary \ref{cor:dp-implies-gen-exp} is to use Markov's inequality. Using this approach, a bound that holds with probability $1-\beta$ will require a polynomial dependence of the sample size on $1/\beta$. While this might lead to a useful bound when the expected empirical loss is small it is less useful in the common scenario when the empirical loss is relatively large. In contrast, our results in Sections \ref{sec:moment-proof} and \ref{sec:concproof} directly imply generalization bounds with logarithmic dependence of the sample size on $1/\beta$. For example, in Theorem \ref{thm:moment-0} we show that for any $\eps,\beta> 0$ and $n \ge O( \ln(1/\beta)/\eps^2)$, the output of an $\eps$-differentially private algorithm $\cA$ satisfies
$\Pr\left[|\calP[\cA(\bS)] - \cE_{\bm S}[\cA(\bS)]| > 2\eps\right]\le \beta$.

\section{Differential Privacy and Preservation of Moments}
\label{sec:moment-proof}

We now prove that if a function~$\bphi$ is output by an
$(\eps,\delta)$-differentially private algorithm~$\cm$ on input of a random dataset
$\bS$ drawn from $\cP^n,$ then the average of $\bphi$ on $\bS,$ that is,
$\cE_{\bS}[\bphi],$ is concentrated around its true expectation $\cP[\bphi].$

The statement we wish to prove is nontrivial due to the apparent dependency between the
function~$\bphi$ and the dataset $\bS$ that arises because $\bphi=\cA(\bS).$ If
instead $\bphi$ was evaluated on a fresh dataset $\bT$ drawn independently of
$\bphi,$ then indeed we would have $\E\cE_{\bT}[\bphi]=\cP[\bphi].$
At a high level, our goal is therefore to resolve the dependency between $\bphi$
and $\bS$ by relating the random variable $\cE_{\bS}[\bphi]$ to the random
variable $\cE_{\bT}[\bphi].$ To argue that these random variables are close
with high probability we relate the moments of $\cE_{\bS}[\bphi]$ to the moments
of $\cE_{\bT}[\bphi].$ The moments of $\cE_{\bT}[\bphi]$ are relatively easy to bound using
standard techniques.

Our proof is easier to execute when $\delta=0$ and we start with this case for
the sake of exposition.

\subsection{Simpler case where $\delta=0$}
Our main technical tool relates the moments of the random variables
that we are interested in.

\begin{lemma}
%\label{lem:moment-0}
\label{lem:moment-bound-dp}
Assume that $\cA$ is an $(\epsilon,0)$-differentially private algorithm
ranging over functions from $\X$ to $[0,1].$
Let $\bS,\bT$ be independent random variables distributed according to $\cP^n.$
For any function $\psi:\X \rar [0,1]$ in the support of $\cm(\bS)$,
\begin{equation}
\E \lb \ldot \cE_{\bS}[\bphi]^k \rcond \bphi = \psi \rb \leq e^{k\eps} \cdot \E \lb \cE_{\bT}[\psi]^k \rb .
\end{equation}
\end{lemma}
\begin{proof}
We use $I$ to denote a $k$-tuple of indices $(i_1,\ldots,i_k) \in [n]^k$ and
use $\bm I$ to denote a $k$-tuple chosen randomly and uniformly from $[n]^k$.
For a data set $T=(y_1,\ldots,y_n)$ we denote by $\Pi_T^I(\psi) = \prod_{j \in
[k]} \psi(y_{i_j})$. We first observe that for any $\psi$,
\begin{equation}
\label{eq:decompose-moment}
\cE_T[\psi]^k
= \E[\Pi_T^{\bm I}(\psi)] .
\end{equation}

For two datasets $S,T \in \uni^n$, let $S_{I \lar T}$ denote the data set in
which for every $j\in [k]$, element $i_j$ in $S$ is replaced with the
corresponding element from $T$. We fix $I$. Note that the random variable
$\bS_{I \lar \bT}$ is distributed according to $\cP^n$ and therefore
\begin{align}
\E \lb \Pi_{\bS}^I(\bphi) \cond \bphi = \psi \rb
& = \E \lb \Pi_{\bS_{I \lar \bT}}^I( \cm(\bS_{I \lar \bT})) \cond \cm(\bS_{I \lar \bT})
= \psi \rb \nonumber \\
& = \E \lb \Pi_{\bT}^I( \cm(\bS_{I \lar \bT})) \cond \cm(\bS_{I
\lar \bT}) = \psi \rb \nonumber \\
& = \int_0^1 \frac{ \pr \lb \Pi_{\bT}^I(
\cm(\bS_{I \lar \bT})) \geq t \mbox{ and } \cm(\bS_{I \lar \bT}) = \psi
\rb}{\pr \lb\cm(\bS_{I \lar \bT}) = \psi \rb} \mathrm{d}t \nonumber \\
& = \int_0^1
\frac{ \pr \lb \Pi_{\bT}^I( \cm(\bS_{I \lar \bT})) \geq t \mbox{ and }
\cm(\bS_{I \lar \bT}) = \psi \rb}{\pr \lb\bphi = \psi \rb} \mathrm{d}t
\label{eq:decondition}
\end{align}

Now for any fixed $t$, $S$ and $T$ consider the event $\Pi_T^I(\cm(S)) \geq t
\mbox{ and } \cm(S) = \psi$ (defined on the range of $\cm$). Data sets $S$ and
$S_{I \lar T}$ differ in at most $k$ elements. Therefore, by the
$\eps$-differential privacy of $\cm$
and Lemma \ref{thm:group-privacy}, the distribution $\cm(S)$ and the distribution $\cm(S_{I \lar T})$ satisfy:
\alequn{
\pr \lb \Pi_T^I(\cm(S_{I \lar T})) \geq t \mbox{ and } \cm(S_{I \lar T}) = \psi \rb \leq e^{k\eps} \cdot \pr \lb \Pi_T^I(\cm(S)) \geq t \mbox{ and } \cm(S) = \psi \rb.}
Taking the probability over $\bS$ and $\bT$ we get:
\alequn{\pr \lb \Pi_{\bT}^I(\cm(\bS_{I \lar \bT})) \geq t \mbox{ and } \cm(\bS_{I \lar \bT}) = \psi \rb \leq e^{k\eps} \cdot \pr \lb \Pi_T^I(\bphi) \geq t \mbox{ and } \bphi = \psi \rb.}

Substituting this into eq.~(\ref{eq:decondition}) we get
\begin{align*}
\E \lb \Pi_{\bS}^I(\bphi) \cond \bphi = \psi \rb   & \leq e^{k\eps} \int_0^1 \frac{ \pr \lb \Pi_{\bT}^I(\bphi) \geq t \mbox{ and } \bphi = \psi \rb}{\pr \lb\bphi = \psi \rb} dt \\
& = e^{k\eps} \E \lb \Pi_{\bT}^I(\bphi) \cond \bphi = \psi \rb \\
& = e^{k\eps} \E \lb  \Pi_{\bT}^I(\psi) \cond \bphi = \psi \rb \\
& = e^{k\eps} \E\lb \Pi_{\bT}^I(\psi)\rb
\end{align*}

Taking the expectation over $\bm I$ and using eq.~(\ref{eq:decompose-moment}) we obtain that
\equn{\E \lb \ldot \cE_{\bS}[\bphi]^k \rcond \bphi = \psi \rb \leq  e^{k\eps}\E \lb \cE_{\bT}[\psi]^k \rb,}
completing the proof of the lemma.
\end{proof}

We now turn our moment inequality into a theorem showing that $\cE_\bS[\bphi]$
is concentrated around the true expectation $\cP[\bphi].$

\begin{theorem}
\label{thm:simulate-direct}
\label{thm:moment-0}
Let $\cm$ be an $\eps$-differentially private algorithm that given a dataset $S$ outputs
a function from $\X$ to $[0,1]$. For any distribution $\calP$ over $\X$ and random variable $\bm S$ distributed according to $\calP^n$ we let $\bm\phi = \cm(\bm S).$ Then for any $\beta> 0, \tau>0$ and $n \ge 12 \ln(4/\beta)/\tau^2$, setting $\eps \le \tau/2$ ensures
$\Pr\left[|\calP[\bm\phi] - \cE_{\bm S}[\bm\phi]| > \tau\right]\le \beta,$ where the probability is over the randomness of $\A$ and $\bm S$.
\end{theorem}
\begin{proof}
Consider an execution of $\cm$ with $\eps=\tau/2$ on a data set $\bS$ of size $n \ge 12 \ln(4/\beta)/\tau^2$.
%We let $p \doteq \cP[\psi]$ and then
By Lemma \ref{lem:moment-bound-by-bernoulli} we obtain that RHS of our bound in Lemma \ref{lem:moment-bound-dp} is at most $e^{\eps k} \mathcal{M}_k[B(n,\cP[\psi])]$. We use Lemma
\ref{lem:markov-bound-by-moment} with $\eps = \tau/2$ and $k =  4
\ln(4/\beta)/\tau$ (noting that the assumption $n \geq 12 \ln(4/\beta)/\tau^2$
ensures the necessary bound on $n$) to obtain that
\remove{
It is easy to see that RHS of our bound in Lemma \ref{lem:moment-bound-dp} is at most $e^{\eps k} \mathcal{M}_k[B(n,\cP[\psi])]$. For $k =  4\ln(4/\beta)/\tau$ we use concentration properties of the $k$-th moment of the sum of $n$ independent variables and Markov's inequality with $\eps =\tau/2$ to obtain that (details appear in the full version \cite{DworkFHPRR14:arxiv})
}
\[
\pr \lb \ldot
\cE_{\bS}[\bphi] \geq \cP[\psi]+\tau \rcond \bphi = \psi \rb \leq \beta/2.
\]
This holds for every $\psi$ in the range of $\cm$ and therefore,
\[
\pr \lb
\cE_{\bS}[\bphi] \geq \cP[\bphi]+\tau \rb \leq \beta/2.
\]
We can apply the same
argument to the function $1-\bphi$ to obtain that
\[
\pr \lb
\cE_{\bS}[\bphi] \leq \cP[\bphi]-\tau \rb \leq \beta/2.
\]
A union bound over the above inequalities implies the claim.
\end{proof}

\subsection{Extension to $\delta > 0$}
We now extend our proof to the case when $\cm$
satisfies $(\eps,\delta)$-differential privacy for sufficiently small but
nonzero $\delta > 0$. The main difficulty in extending the previous proof is that the
condition $\{\bphi=\psi\}$ appearing in Lemma~\ref{lem:moment-bound-dp} may have
arbitrarily small probability. A simple extension of the previous proof would lead
to an error of $\delta/\Pr[\bphi=\psi].$ We avoid this issue by using a more
carefully chosen condition. Specifically, instead of restricting $\bphi$ to be
equal to a particular function $\psi$, we only constrain $\cP[\bphi]$ to be in
a certain interval of length~$\tau.$ This conditioning still gives us enough
information about $\bphi$ in order to control $\cE_\bT[\bphi],$ while allowing
us to ignore events of exceedingly small probability.
\begin{theorem}
\label{thm:simulate-direct-delta}
Let $\cm$ be an $(\eps,\delta)$-differentially private algorithm that given a
dataset $S$ outputs a function from $\X$ to $[0,1]$. For any distribution
$\calP$ over $\X$ and random variable $\bm S$ distributed according to
$\calP^n$ we let $\bm\phi = \cm(\bm S).$ Then for any $\beta> 0, \tau>0$ and
$n \ge 48 \ln(4/\beta)/\tau^2$, setting $\eps \le \tau/4$ and $\de = \exp(- 4
\cdot \ln(8/\beta)/\tau )$ ensures $\Pr\left[|\calP[\bm\phi] - \cE_{\bm
S}[\bm\phi]| > \tau\right]\le \beta,$ where the probability is over the
randomness of $\A$ and $\bm S$.
\end{theorem}
\begin{proof}
We use the notation from the proof of Theorem \ref{thm:simulate-direct} and
consider an execution of $\cm$ with $\eps$ and $\delta$ satisfying the conditions of the theorem.

Let $L = \lceil 1/\tau \rceil$. For a value $\ell \in [L]$ we use $B_\ell$ to denote the interval set $[(\ell-1)\tau,\ell\tau]$.

%$\cP[\bphi] \in B_\ell$ to denote the event, in the space of the outputs of $A$, that $\cP(\phi_m) \in [(\ell-1)\tau,\ell\tau)$. Note that given an output of $A$, its membership in $\cP[\bphi] \in B_\ell$ can be determined by simulating $\B$.

We say that $\ell \in [L]$ is {\em heavy} if $\pr \lb  \cP[\bphi] \in B_\ell \rb \geq \beta/(4L)$ and we say that $\ell$ is {\em light} otherwise. The key claim that we prove is an upper bound on the $k$-th moment of $\cE_{\bS}[\bphi]$ for heavy $\ell$'s:
\equ{\E \lb \ldot \cE_{\bS}[\bphi]^k \rcond \cP[\bphi] \in B_\ell \rb \leq e^{k\eps} \cdot \mathcal{M}_k[B(n, \tau \ell)] + \delta e^{(k-1)\eps} \cdot 4L/\beta \label{eq:moment-bound-dp-delta}.}

%To prove this claim
We use the same decomposition of the $k$-th moment as before:
\equn{ \E\lb \ldot \cE_{\bS}[\bphi]^k \rcond \cP[\bphi] \in B_\ell \rb =  \E \lb \ldot \Pi_{\bS}^{\bm I}(\bphi)  \rcond \cP[\bphi] \in B_\ell \rb .}

Now for a fixed $I \in [n]^k$, exactly as in eq.~(\ref{eq:decondition}), we
obtain \alequ{ \E \lb \Pi_{\bS}^I(\bphi) \cond \cP[\bphi] \in B_\ell \rb   =
\int_0^1 \frac{ \pr\lb \Pi_{\bT}^I(\cm(\bS_{I \lar \bT})) \geq t \mbox{ and }
\cP[\cm(\bS_{I \lar \bT})] \in B_\ell\rb}{\pr \lb \cP[\bphi] \in B_\ell \rb}
dt \label{eq:decondition-delta}
}
Now for fixed values of $t$, $S$ and $T$ we consider the event
$\Pi_T^I(\cm(S)) \geq t \mbox{ and } \cP[\cm(S)]\in B_\ell$ defined on the
range of $\cm$. Datasets $S$ and $S_{I \lar T}$ differ in at most $k$
elements. Therefore, by the $(\eps,\delta)$-differential privacy of $\cm$ and
Lemma \ref{thm:group-privacy}, the distribution over the output of $\cm$ on input $S$ and the distribution over the output of $\cm$ on input $S_{I \lar T}$ satisfy:
\alequn{\pr \lb \Pi_T^I(\cm(S_{I \lar T})) \geq t \mbox{ and } \cP[\cm(S_{I \lar T})] \in B_\ell \rb \\ \leq e^{k\eps} \cdot \pr\lb \Pi_T^I(\cm(S)) \geq t \mbox{ and } \cP[\cm(S)] \in B_\ell \rb + e^{(k-1)\eps}\delta.}
Taking the probability over $\bS$ and $\bT$ and substituting this into eq.~\eqref{eq:decondition-delta} we get
\alequn{ \E \lb \Pi_{\bS}^I(\bphi) \cond \cP[\bphi] \in B_\ell \rb  & \leq e^{k\eps} \int_0^1 \frac{ \pr \lb \Pi_{\bT}^I(\bphi) \geq t \mbox{ and } \cP[\bphi] \in B_\ell \rb}{\pr \lb\cP[\bphi] \in B_\ell\rb} dt + \frac{e^{(k-1)\eps}\delta}{\pr \lb\cP[\bphi] \in B_\ell\rb} \\
& = e^{k\eps} \E \lb \Pi_{\bT}^I(\bphi) \cond \cP[\bphi] \in B_\ell \rb + \frac{e^{(k-1)\eps}\delta}{\pr \lb\cP[\bphi] \in B_\ell\rb} }

Taking the expectation over $\bm I$ and using eq.~\eqref{eq:decompose-moment} we obtain:
\equ{\E \lb \ldot \cE_{\bS}[\bphi]^k \rcond \cP[\bphi] \in B_\ell \rb \leq e^{k\eps} \E \lb \ldot \cE_{\bT}[\bphi]^k \rcond \cP[\bphi] \in B_\ell \rb + \frac{e^{(k-1)\eps}\delta}{\pr \lb \cP[\bphi] \in B_\ell\rb} .\label{eq:moment-with-weight}}

Conditioned on $\cP[\bphi] \in B_\ell$,  $\cP[\bphi] \leq \tau\ell$ and therefore by Lemma \ref{lem:moment-bound-by-bernoulli}, $$\E\lb\ldot \cE_{\bT}[\bphi]^k \rcond \cP[\bphi] \in B_\ell \rb \leq \mathcal{M}_k[B(n, \tau \ell)].$$ In addition, by our assumption, $\ell$ is heavy, that is
$\pr \lb \cP[\bphi] \in B_\ell\rb \geq \beta/(4L)$. Substituting these values into eq.~\eqref{eq:moment-with-weight} we obtain the claim in eq.~\eqref{eq:moment-bound-dp-delta}.

As before, we use Lemma \ref{lem:markov-bound-by-moment} with $\eps = \tau/2$ and $k = 4 (\tau \ell) \ln(4/\beta)/\tau = 4 \ell \ln(4/\beta)$ (noting that condition $n \geq 12 \ln(4/\beta)/\tau^2$ ensures the necessary bound on $n$) to obtain that
\equ{\pr \lb \ldot \cE_{\bS}[\bphi] \geq \tau\ell+\tau \rcond \cP[\bphi] \in B_\ell \rb \leq \beta/2 + \frac{\delta e^{(k-1)\eps} \cdot 4L}{\beta (\tau \ell +\tau)^k}, \label{eq:bound-error-from-delta}}
Using condition $\delta = \exp(- 2 \cdot \ln(4/\beta)/\tau )$ and inequality $\ln(x) \leq x/e$  (for $x>0$) we obtain
\alequn{ \frac{\delta e^{(k-1)\eps} \cdot 4L}{\beta ((\ell+1)\tau)^k} & \leq
\frac{\delta\cdot e^{2 \ln(4/\beta)} \cdot 4/\tau}{\beta e^{4 \ln((\ell+1)\tau) \cdot \ell \ln(4/\beta)} }\\
& \leq
\frac{\delta\cdot e^{4 \ln(4/\beta)}}{\tau \cdot e^{4 \ln((\ell+1)\tau) \cdot \ell \ln(4/\beta)} } \cdot \frac{\beta}{4} \\
&\leq \delta \cdot  \exp\lp 4 \ln(1/((\ell+1)\tau)) \cdot \ell \ln(4/\beta)  + 4 \ln(4/\beta) +\ln(1/\tau) \rp   \cdot \frac{\beta}{4} \\
&\leq \delta \cdot  \exp\lp \frac{4}{e} \cdot \frac{1}{(\ell+1)\tau} \cdot \ell \ln(4/\beta) + 4 \ln(4/\beta) +\ln(1/\tau) \rp   \cdot \frac{\beta}{4} \\
&\leq \delta \cdot  \exp\lp \frac{4}{e} \cdot \ln(4/\beta)/\tau + 4 \ln(4/\beta) +\ln(1/\tau) \rp   \cdot \frac{\beta}{4}
\\
&\leq \delta \cdot  \exp\lp 2 \cdot \ln(4/\beta)/\tau \rp   \cdot \frac{\beta}{4}
 \leq \beta/4.}
Substituting this into eq.~\eqref{eq:bound-error-from-delta} we get
\equn{\pr \lb \ldot \cE_{\bS}[\bphi] \geq \tau\ell+\tau \rcond \cP[\bphi] \in B_\ell \rb \leq 3\beta/4.}

Note that, conditioned on $\cP[\bphi] \in B_\ell$, $\cP[\bphi] \geq \tau(\ell-1)$, and therefore
\equn{\pr \lb \ldot \cE_{\bS}[\bphi] \geq \cP[\bphi] + 2\tau \rcond \cP[\bphi] \in B_\ell \rb \leq 3\beta/4 .}
This holds for every heavy $\ell \in [L]$ and therefore,
\alequn{\pr \lb \cE_{\bS}[\bphi] \geq \cP[\bphi]+2\tau \rb & \leq 3\beta/4 + \sum_{\ell \in [L] \mbox { is light}} \pr\lb \cP[\bphi] \in B_\ell \rb \\ & \leq 3\beta/4 + L \beta/(4L) = \beta .}
Apply the same argument to $1-\bphi$ and use a union bound. We obtain the claim
after rescaling $\tau$ and $\beta$ by a factor~$2.$

\end{proof}

\section{Beyond statistical queries}
\label{sec:concproof}

The previous section dealt with statistical queries.
A different way of looking at our results is to define for each
function~$\psi$ a set $R(\psi)$ containing all datasets~$S$ such that $\psi$ is far from the
correct value $\cP[\psi]$ on $S.$ Formally,
$R(\psi) = \{ S \colon
\left|\cE_S[\psi]-\cP[\psi]\right|>\tau\}.$
Our results showed that if $\bphi = \cA(\bS)$ is the output of a differentially
private algorithm $\cA$ on a random dataset $\bS,$ then $\Pr[ \bS\in
R(\bphi)]$ is small.

Here we prove a broad generalization that allows the differentially private
algorithm to have an arbitrary output space~$Z$. The same conclusion holds for any
collection of sets $R(y)$ where $y\in Z$ provided that $\Pr[\bS\in R(y)]$ is
small for all $y\in Z.$

\begin{theorem}
\label{thm:concentrated-divergence}
Let $\cm$ be an $(\eps,0)$-differentially private algorithm with range $Z$.  For a distribution $\cP$ over $\uni$,
 let $\bS$ be a random variable drawn from $\cP^n$.
Let $\bY = \cm(\bm S)$ be the random variable output by $\cm$ on input $\bS$.
For each element $y \in Z$ let
$R(y)\subseteq\uni^n$ be some subset of datasets and assume that $\max_y \pr[\bS \in R(y)]\le\beta.$
Then, for $\eps \leq \sqrt{\frac{\ln(1/\beta)}{2n}}$ we have
$\pr[\bS \in R(\bm Y) ] \leq 3\sqrt{\beta}$.
\end{theorem}
\begin{proof} Fix $y \in Z$. We first observe that by Jensen's inequality,
\[
\E_{S \sim \cP^n} [\ln(\pr[\bY = y \cond \bS = S])] \leq \ln \lp \E_{S \sim
\cP^n} [\pr[\bY=y \cond \bS  = S]] \rp = \ln(\pr[\bY=y]).
\]
Further, by
definition of differential privacy, for two databases $S,S'$ that differ in a
single element, $$\pr[\bY=y \cond \bS = S] \leq e^\eps \cdot \pr[\bY=y \cond
\bS = S'].$$

Now consider the function $g(S) = \ln\lp\frac{\pr[\bY=y \cond \bS = S]}{\pr[\bY=y]}\rp$. By the properties above we have that $\E[g(\bS)] \leq \ln(\pr[\bY=y]) - \ln(\pr[\bY=y]) = 0$ and $|g(S) - g(S')| \leq \eps.$ This, by McDiarmid's inequality (Lemma \ref{lem:mcdiarmid}), implies that for any $t > 0$,
\begin{equation} \label{eq:main-conc}
\pr[g(\bS) \geq \eps t ] \leq e^{-2t^2/n}.
\end{equation}
%or equivalently,
%$$\pr_{S \sim \cP^n}\lb\ln\lp\frac{\pr[\bY=y \cond \bS = S]}{\pr[\bY=y]}\rp \geq t \rb \leq e^{-2t^2/n}.$$

For an integer $i \geq 1$ let $$B_i \doteq \left\{ S \ \left| \  \eps \sqrt{n\ln(2^i/\beta)/2} \leq g(S) \leq \eps \sqrt{n\ln(2^{i+1}/\beta)/2}\right. \right\}$$ and let $B_0 \doteq \{ S \cond g(S) \leq  \eps \sqrt{n\ln(2/\beta)/2} \}.$

By inequality \eqref{eq:main-conc} we have that for $i\geq 1$, $\pr[g(\bS) \geq \eps\sqrt{n\ln(2^i/\beta)/2}] \leq \beta/2^i$. Therefore, for all $i \geq 0$, $$\pr[\bS \in B_i \cap R(y)] \leq \beta/2^i,$$ where the case of $i=0$ follows from the assumptions of the lemma.

By Bayes' rule, for every $S \in B_i$,
$$\frac{\pr[\bS=S \cond \bY = y]}{\pr[\bS=S]} = \frac{\pr[\bY=y \cond \bS = S]}{\pr[\bY=y]} = \exp(g(S)) \leq \exp\lp\eps\sqrt{n\ln(2^{i+1}/\beta)/2}\rp.$$
Therefore,
\alequ{\pr[\bS \in B_i \cap R(y) \cond \bY=y ] &= \sum_{S\in B_i\cap R(y)} \pr[\bS =S \cond \bY = y]  \nonumber\\
&\leq \exp\lp\eps\sqrt{n\ln(2^{i+1}/\beta)/2}\rp \cdot \sum_{S\in B_i\cap R(y)} \pr[\bS =S]\nonumber\\
&= \exp\lp\eps\sqrt{n\ln(2^{i+1}/\beta)/2}\rp \cdot \pr[\bS \in B_i \cap R(y)] \nonumber\\
&\leq \exp\lp\eps\sqrt{n\ln(2^{i+1}/\beta)/2} - \ln(2^i/\beta)\rp. \label{eq-i-bucket-bound}
}
The condition $\eps \leq \sqrt{\frac{\ln(1/\beta)}{2n}}$ implies that
\begin{align*}
\eps\sqrt{\frac{n\ln(2^{i+1}/\beta)}{2}} - \ln(2^i/\beta)
&\leq \sqrt{\frac{\ln(1/\beta) \ln(2^{i+1}/\beta)}{4}} - \ln(2^i/\beta) \\
& \leq \frac{\ln(2^{(i+1)/2}/\beta)}{2} - \ln(2^i/\beta) = -\ln\lp \frac{2^{(3i-1)/4}}{\sqrt{\beta}}\rp
\end{align*}
Substituting this into inequality \eqref{eq-i-bucket-bound}, we get $$\pr[\bS \in B_i \cap R(y) \cond \bY=y ] \leq \frac{\sqrt{\beta}}{2^{(3i-1)/4}} .$$

Clearly, $\cup_{i\geq 0} B_i = \X^{[n]}$. Therefore
$$\pr[\bS \in R(y) \cond \bY = y] = \sum_{i\geq 0} \pr[\bS \in B_i \cap R(y) \cond \bY = y] \leq \sum_{i\geq 0} \frac{\sqrt{\beta}}{2^{(3i-1)/4}} = \sqrt{\beta} \cdot \frac{2^{1/4}}{1-2^{-3/4}} \leq 3\sqrt{\beta}.
$$
Finally, let $\cal Y$ denote the distribution of $\bY$. Then,
$$\pr[\bS \in R(\bm Y) ] = \E_{y \sim \cal Y}[\pr[\bS \in R(y) \cond \bY = y ]] \leq 3\sqrt{\beta}.$$

\end{proof}
Our theorem gives a result for statistical queries that achieves the
same bound as our earlier result in Theorem~\ref{thm:moment-0} up to constant
factors in the parameters.

\begin{corollary}
\label{cor:strong bound for counts}
Let $\cm$ be an $\eps$-differentially private algorithm that outputs a function from $\uni$ to $[0,1]$. For a distribution $\cP$ over $\X$, let $\bS$ be a random variable distributed according to $\cP^n$ and let $\bphi = \cm(\bS).$ Then for any $\tau>0$, setting $\eps \le
\sqrt{\tau^2-\ln(2)/2n}$ ensures
$\pr\left[|\calP[\bphi] - \cE_{\bS}[\bphi]| > \tau\right]\le
3\sqrt{2}e^{-\tau^2 n}.$
\end{corollary}
\begin{proof}
By the Chernoff bound, for any fixed query function $\psi:\X\rar [0,1]$, $$\pr[|\calP[\psi] - \cE_{\bS}[\psi]| \geq \tau] \leq 2e^{-2\tau^2n}.$$
Now, by Theorem~\ref{thm:concentrated-divergence} for $R(\psi) =  \set{ S \in
\uni^n \cond |\calP[\psi] - \cE_{\bS}[\psi]| > \tau}$, $\beta =
2e^{-2\tau^2n}$ and any $\eps \leq \sqrt{\tau^2-\ln(2)/2n}$,
$$\pr\left[|\calP[\bphi] - \cE_{\bS}[\bphi]| > \tau\right]\le
3\sqrt{2}e^{-\tau^2 n}.$$
\end{proof}

\section{Applications}
\label{sec:applications}
To obtain algorithms for answering adaptive statistical queries we first note that if for a query function $\psi$ and a dataset $S$, $|\cP[\psi] - \cE_S[\psi]| \leq \tau/2$ then we can use an algorithm that outputs a value $v$ that is $\tau/2$-close to $\cE_S[\psi]$ to obtain a value that is $\tau$-close to $\cP[\psi]$. Differentially private algorithms that for a given dataset $S$ and an adaptively chosen sequence of queries $\phi_1,\ldots,\phi_m$  produce a value close to $\cE_S[\phi_i]$ for each query $\phi_i\colon \uni \to [0,1]$ have been the subject of intense investigation in the differential privacy literature (see \cite{DR14} for an overview). Such queries are usually referred to as (fractional) {\em counting queries} or {\em linear queries} in this context. This allows us to obtain statistical query answering algorithms by using various known differentially private algorithms for answering counting queries.

The results in Sections \ref{sec:moment-proof} and \ref{sec:concproof} imply that $|\cP[\psi] - \cE_\bS[\psi]| \leq \tau$ holds with high probability whenever $\psi$ is generated by a differentially private
algorithm $\cM$. This might appear to be inconsistent with our application since there the
queries are generated by an arbitrary (possibly adversarial) adaptive analyst and we can only guarantee
that the query answering algorithm is differentially private. The
connection comes from the following basic fact about differentially private
algorithms:

\begin{fact}[Postprocessing Preserves Privacy (see e.g.~\cite{DR14})]
Let $\cM:\uni^n\rightarrow \cO$ be an $(\epsilon,\delta)$ differentially
private algorithm with range $\cO$, and let $\cal F:\cO\rightarrow \cO'$ be an
arbitrary randomized algorithm. Then ${\cal F} \circ {\cM}:{\X}^{n}\rightarrow \cO'$ is
$(\epsilon,\delta)$-differentially private.
\end{fact}

Hence, an \emph{arbitrary} adaptive analyst $\cm$ is guaranteed to generate
queries in a manner that is differentially private in $\bS$ so long as the
only access that she has to $\bS$ is through a differentially private
query answering algorithm $\cM$.
We also note that the bounds we state here give the probability of correctness
for each individual answer to a query, meaning that the error probability $\beta$ is for each query
$\phi_i$ and not for all queries at the same time. The bounds we state in Section \ref{sec:our-results}
hold with high probability for all $m$ queries and to obtain them from the bounds in this
section, we apply the union bound by setting $\beta = \beta'/m$ for some small $\beta'$.

We now highlight a few applications of differentially private algorithms for answering counting queries to our problem.

\subsection{Laplacian Noise Addition}
\newcommand{\Laplace}{{\sf Laplace}\xspace}
\newcommand{\PMW}{{\sf PMW}\xspace}

The Laplacian Mechanism on input of a dataset~$S$ answers $m$ adaptively chosen queries
$\phi_1,\dots,\phi_m$ by responding with $\phi_i(S) + \mathrm{Lap}(0,\sigma)$
when given query $\phi_i.$ Here, $\mathrm{Lap}(0,\sigma)$ denotes a Laplacian
random variable of mean $0$ and scale $\sigma.$ For suitably chosen $\sigma$
the algorithm has the following guarantee.

\begin{theorem}[Laplace]
\label{thm:laplace-0}
Let $\tau,\beta,\epsilon>0$ and define
\[
n_L(\tau,\beta,\epsilon,m) = \frac{m\log(1/\beta)}{\epsilon\tau}\,.
\]
\[
n_L^\delta(\tau,\beta,\epsilon,\delta,m) = \frac{\sqrt{m\log(1/\delta)}\log(1/\beta)}{\epsilon\tau}\,.
\]
There is computationally efficient algorithm called \Laplace which on input
of a data set $S$ of size $n$ accepts any sequence of $m$ adaptively chosen
functions $\phi_1,\dots,\phi_m \in \uni^{[0,1]}$ and returns estimates
$a_1,\dots,a_m$ such that for every $i\in[m]$ we have
$\Pr\left[\left|\cE_S[\phi_i]- a_i\right|>\tau\right]\le\beta$. To achieve
this guarantee under $(\epsilon,0)$-differential privacy, it requires $n\ge C
n_L(\tau,\beta,\epsilon,m)$, and to achieve this guarantee under
$(\epsilon,\delta)$-differential privacy, it requires $n \ge C
n_L^\delta(\tau,\beta,\epsilon,\delta,m)$ for sufficiently large constant
$C$.
\end{theorem}

Applying our main generalization bound for $(\epsilon,0)$-differential privacy
directly gives the following corollary.

\begin{corollary}
\label{cor:laplace-0}
Let $\tau,\beta>0$ and define
\[
n_L(\tau,\beta,m) = \frac{m\log(1/\beta)}{\tau^2}\,.
\]
There is a computationally efficient algorithm which on input of a data set
$S$ of size $n$ sampled from $\cP^n$ accepts any sequence of $m$ adaptively
chosen functions $\phi_1,\dots,\phi_m \in \uni^{[0,1]}$ and returns estimates
$a_1,\dots,a_m$ such that for every $i\in[m]$ we have
$\Pr\left[\left|\cP[\phi_i]- a_i\right|>\tau\right]\le\beta$ provided that
$n\ge C n_L(\tau,\beta,m)$ for sufficiently large constant $C$.
\end{corollary}

\begin{proof}
We apply Theorem~\ref{thm:simulate-direct} with $\epsilon=\tau/2$ and plug
this choice of $\epsilon$ into the definition of $n_L$ in
Theorem~\ref{thm:laplace-0}. We note that the stated lower bound on $n$
implies the lower bound required by Theorem~\ref{thm:simulate-direct}.
\end{proof}

%The Laplacian Mechanism also satisfies $(\epsilon,\delta)$-differential
%privacy for a different choice of~$\sigma$.
%
%\begin{theorem}[Laplace]
%\label{thm:laplace-delta}
%Let $\tau,\beta,\epsilon,\delta>0$ and define
%\[
%n_0(\tau,\beta,\epsilon,\delta,m) = \frac{\sqrt{m\log(1/\delta)}\log(1/\beta)}{\epsilon\tau}\,.
%\]
%There is a computationally efficient $(\epsilon,\delta)$-differentially
%private algorithm called \Laplace which on input of a data set $S$ of size
%$n$ accepts any sequence of $m$ adaptively chosen functions $\phi_1,\dots,\phi_m \in \uni^{[0,1]}$
%and returns estimates $a_1,\dots,a_m$ such that for every $i\in[m]$ we have
%$\Pr\left[\left|\cE_S[\phi_i]- a_i\right|>\tau\right]\le\beta$
%provided that $n\ge C n_0(\tau,\beta,\epsilon,\delta,m)$ for sufficiently
%large constant $C.$
%\end{theorem}

The corollary that follows the $(\epsilon,\delta)$ bound gives a quadratic
improvement in $m$ compared with Corollary~\ref{cor:laplace-0} at the expense
of a slightly worse dependence on $\tau$ and $1/\beta.$

\begin{corollary}
\label{cor:laplace-delta}
Let $\tau,\beta>0$ and define
\[
n_L^\delta(\tau,\beta,m) = \frac{\sqrt{m}\log^{1.5}(1/\beta)}{\tau^{2.5}}\,.
\]
There is a computationally efficient algorithm which on input of a data set
$S$ of size $n$ sampled from $\cP^n$ accepts any sequence of $m$ adaptively
chosen functions $\phi_1,\dots,\phi_m \in \uni^{[0,1]}$ and returns estimates
$a_1,\dots,a_m$ such that for every $i\in[m]$ we have
$\Pr\left[\left|\cP[\phi_i]- a_i\right|>\tau\right]\le\beta$ provided that
$n\ge C n_L^\delta(\tau,\beta,m)$ for sufficiently large constant $C.$
\end{corollary}

\begin{proof}
We apply Theorem~\ref{thm:simulate-direct-delta} with $\epsilon=\tau/2$
and $\delta=\exp(-4\ln(8/\beta)/\tau).$ Plugging these parameters into the
 definition of $n_L^\delta$ in Theorem~\ref{thm:laplace-0} gives the stated lower
bound on $n.$ We note that the stated lower bound on $n$ implies the lower
bound required by Theorem~\ref{thm:simulate-direct-delta}.
\end{proof}

\subsection{Multiplicative Weights Technique}

The private multiplicative weights algorithm~\cite{HardtR10} achieves an
exponential improvement in $m$ compared with the Laplacian mechanism. The
main drawback is a running time that scales linearly with the domain size in
the worst case and is therefore not computationally efficient in general.

%\begin{theorem}[Private Multiplicative Weights]
%\label{thm:pmw-0}
%Let $\tau,\beta,\epsilon>0$ and define
%\[
%n_{MW}(\tau,\beta,\epsilon) = \frac{\log(|\uni|)\log(n\log(|\uni|)/\beta)}{\epsilon\tau^3}\,.
%\]
%\[
%n_{MW}^\delta(\tau,\beta,\epsilon,\delta) = \frac{\sqrt{\log(|\uni|)\log(1/\delta)}\log(n/\beta)}{\epsilon\tau^2}\,.
%\]
%There is algorithm called \PMW which on input of a data set $S$ of size $n$
%accepts any sequence of $m$ adaptively chosen functions $\phi_1,\dots,\phi_m
%\in \uni^{[0,1]}$ and returns estimates $a_1,\dots,a_m$ such that for every
%$i\in[m]$ we have $\Pr\left[\left|\cE_S[\phi_i]-
%a_i\right|>\tau\right]\le\beta$. To achieve this guarantee under
%$(\epsilon,0)$ differential privacy, it requires that $n\ge C
%n_{MW}(\tau,\beta,\epsilon)$ and to achieve it under
%$(\epsilon,\delta)$-differential privacy it requires $n\ge C
%n_{MW}^\delta(\tau,\beta,\epsilon,\delta)$ for sufficiently large constant
%$C.$
%\end{theorem}
\begin{theorem}[Private Multiplicative Weights]
\label{thm:pmw-0} Let $\tau,\beta,\epsilon>0$ and define
\[
n_{MW}(\tau,\beta,\epsilon) = \frac{\log(|\uni|)\log(1/\beta)}{\epsilon\tau^3}\,.
\]
\[
n_{MW}^\delta(\tau,\beta,\epsilon,\delta) = \frac{\sqrt{\log(|\uni|)\log(1/\delta)}\log(1/\beta)}{\epsilon\tau^2}\,.
\]
There is algorithm called \PMW which on input of a data set $S$ of size $n$
accepts any sequence of $m$ adaptively chosen functions $\phi_1,\dots,\phi_m
\in \uni^{[0,1]}$ and with probability at least $1-(n\log|\uni|)\beta$ returns
estimates $a_1,\dots,a_m$ such that for every $i\in[m]$ we have
$\Pr\left[\left|\cE_S[\phi_i]- a_i\right|>\tau\right]\le\beta$. To achieve
this guarantee under $(\epsilon,0)$ differential privacy, it requires that
$n\ge C n_{MW}(\tau,\beta,\epsilon)$ and to achieve it under
$(\epsilon,\delta)$-differential privacy it requires $n\ge C
n_{MW}^\delta(\tau,\beta,\epsilon,\delta)$ for sufficiently large constant
$C.$
\end{theorem}

\begin{corollary}
\label{cor:pmw-0}
Let $\tau,\beta>0$ and define
\[
n_{MW}(\tau,\beta) = \frac{\log(|\uni|)\log(1/\beta)}{\tau^4}\,.
\]
There is an algorithm which on input of a data set $S$ of size $n$ sampled
from $\cP^n$ accepts any sequence of $m$ adaptively chosen functions
$\phi_1,\dots,\phi_m \in \uni^{[0,1]}$  and with probability at least
$1-(n\log|\uni|)\beta$ returns estimates $a_1,\dots,a_m$ such that for every
$i\in[m]$ we have $\Pr\left[\left|\cP[\phi_i]-
a_i\right|>\tau\right]\le\beta$ provided that $n\ge C n_{MW}(\tau,\beta)$ for
sufficiently large constant $C.$
\end{corollary}

\begin{proof}
We apply Theorem~\ref{thm:simulate-direct} with $\epsilon=\tau/2$ and plug
this choice of $\epsilon$ into the definition of $n_{MW}$ in
Theorem~\ref{thm:pmw-0}. We note that the stated lower bound on $n$ implies
the lower bound required by Theorem~\ref{thm:simulate-direct}.
\end{proof}

%\PMW also satisfies $(\epsilon,\delta)$-differential privacy with the
%following quantitative bound.
%
%\begin{theorem}[Private Multiplicative Weights]
%\label{thm:pmw-delta}
%Let $\tau,\beta,\epsilon,\delta>0$ and define
%\[
%n_0(\tau,\beta,\epsilon,\delta) = \frac{\sqrt{\log(|\uni|)\log(1/\delta)}\log(n/\beta)}{\epsilon\tau^2}\,.
%\]
%There is an $(\epsilon,\delta)$-differentially
%private algorithm called \PMW which on input of a data set $S$ of size
%$n$ accepts any sequence of $m$ adaptively chosen functions $\phi_1,\dots,\phi_m \in \uni^{[0,1]}$
%and returns estimates $a_1,\dots,a_m$ such that for every $i\in[m]$ we have
%$\Pr\left[\left|\cE_S[\phi_i]- a_i\right|>\tau\right]\le\beta,$
%provided that $n\ge C n_0(\tau,\beta,\epsilon,\delta)$ for sufficiently
%large constant $C.$
%\end{theorem}

Under $(\epsilon,\delta)$ differential privacy we get the following corollary
that improves the dependence on $\tau$ and $\log|\uni|$ in
Corollary~\ref{cor:pmw-0} at the expense of a slightly worse dependence on
$\beta.$

\begin{corollary}
\label{cor:pmw-delta}
Let $\tau,\beta>0$ and define
\[
n_{MW}^\delta(\tau,\beta) = \frac{\sqrt{\log(|\uni|)}\log(1/\beta)^{3/2}}{\tau^{3.5}}\,.
\]
There is an algorithm which on input of a data set $S$ of size $n$ sampled
from $\cP^n$ accepts any sequence of $m$ adaptively chosen functions
$\phi_1,\dots,\phi_m \in \uni^{[0,1]}$  and with probability at least
$1-(n\log|\uni|)\beta$ returns estimates $a_1,\dots,a_m$ such that for every
$i\in[m]$ we have $\Pr\left[\left|\cP[\phi_i]-
a_i\right|>\tau\right]\le\beta$ provided that $n\ge C
n_{MW}^\delta(\tau,\beta)$ for sufficiently large constant $C.$
\end{corollary}

\begin{proof}
We apply Theorem~\ref{thm:simulate-direct-delta} with $\epsilon=\tau/2$ and
$\delta=\exp(-4\ln(8/\beta)/\tau).$ Plugging these parameters into the
definition of $n_{MW}^\delta$ in Theorem~\ref{thm:pmw-0} gives the stated
lower bound on $n$. We note that the stated lower bound on $n$ implies the
lower bound required by Theorem~\ref{thm:simulate-direct-delta}.
\end{proof}

\subsection{Sparse Vector Technique}
\newcommand{\SPARSE}{{\sf SPARSE}\xspace}

In this section we give a computationally efficient technique for answering
exponentially many queries $\phi_1,\dots,\phi_m$ in the size of the data set
$n$ so long as they are chosen using only $o(n)$ rounds of adaptivity. We say
that a sequence of queries $\phi_1,\ldots,\phi_m \in \uni^{[0,1]}$, answered with numeric
values $a_1,\ldots,a_m$ is generated with $r$ rounds of adaptivity if there
are $r$ indices $i_1,\ldots,i_r$ such that the procedure that generates the
queries as a function of the answers can be described by $r+1$ (possibly
randomized) algorithms $f_0,f_1,\ldots,f_r$ satisfying:
\[(\phi_1,\ldots,\phi_{i_1-1}) = f_0(\emptyset)\]
\[(\phi_{i_1},\ldots,\phi_{i_2-1}) = f_1((\phi_1,a_1),\ldots,(\phi_{i_1-1},a_{i_1-1})) \]
\[(\phi_{i_2},\ldots,\phi_{i_3-1}) = f_2((\phi_1,a_1),\ldots,(\phi_{i_2-1},a_{i_2-1}))  \]
\[ \vdots \]
\[  (\phi_{i_r},\ldots,\phi_{m}) = f_r((\phi_1,a_1),\ldots,(\phi_{i_r-1},a_{i_r-1})) \]

We build our algorithm out of a differentially private algorithm called
\SPARSE that takes as input an adaptively chosen sequence of queries together
with \emph{guesses of the answers to those queries}. Rather than always
returning numeric valued answers, it compares the error of our guess to a
\emph{threshold} $T$ and returns a numeric valued answer to the query only if
(a noisy version of) the error of our guess was above the given threshold.
\SPARSE is computationally efficient, and has the remarkable property that
its accuracy has polynomial dependence only on the number of queries for
which the error of our guesses are close to being above the threshold.

\begin{theorem}[Sparse Vector $(\epsilon,0)$]
\label{thm:sparse-0} Let $\tau,\beta,\epsilon > 0$ and define
\[
n_{SV}(\tau,\beta,\epsilon) = \frac{9r\left(\ln(4/\beta)\right)}{\tau\epsilon}\,.
\]
\[
n_{SV}^\delta(\tau,\beta,\epsilon,\delta) = \frac{\left(\sqrt{512}+1\right)\sqrt{r\ln(2/\delta)}\left(\ln(4/\beta)\right)}{\tau\epsilon}\,.
\]
There is an algorithm called \SPARSE parameterized by a real valued threshold
$T$, which on input of a data set $S$ of size $n$ accepts any sequence of $m$
adaptively chosen queries together with guesses at their values $g_i \in
\mathbb{R}$: $(\phi_1,g_1),\dots,(\phi_m,g_m)$ and returns answers
$a_1,\dots,a_m \in \{\bot\} \cup \mathbb{R}$. It has the property that for
all $i\in[m]$, with probability $1-\beta$: if $a_i = \bot$ then
$|\cE_S[\phi_i] - g_i|\leq T + \tau$ and if $a_i \in \mathbb{R},
|\cE_S[\phi_i] - a_i|\leq \tau$. To achieve this guarantee under
$(\epsilon,0)$-differential privacy it requires $n\ge
n_{SV}(\tau,\beta,\epsilon)$ and to achieve this guarantee under
$(\epsilon,\delta)$-differential privacy, it requires $n \ge
n_{SV}^\delta(\tau,\beta,\epsilon,\delta)$. In either case, the algorithm
also requires that $|\{i : |\cE_S[\phi_i]-g_i| \geq T - \tau\}| \leq r$. (If
this last condition does not hold, the algorithm may halt early and stop
accepting queries)
\end{theorem}

%\begin{theorem}[Sparse Vector $(\epsilon,\delta)$]
%\label{thm:sparse-delta} Let $\tau,\beta,\epsilon,\delta>0$ and define
%\[
%n_0(\tau,\beta,\epsilon,\delta) = \frac{\left(\sqrt{512}+1\right)\sqrt{r\ln(2/\delta)}\left(\ln(4/\beta)\right)}{\tau\epsilon}\,.
%\]
%There is an $(\epsilon,\delta)$-differentially private algorithm called
%\SPARSE parameterized by a real valued threshold $T$, which on input of a
%data set $S$ of size $n$  accepts any sequence of $m$ adaptively chosen
%queries together with guesses at their values $g_i \in \mathbb{R}$:
%$(\phi_1,g_1),\dots,(\phi_m,g_m)$ and returns answers $a_1,\dots,a_m \in
%\{\bot\} \cup \mathbb{R}$. It has the property that for all $i\in[m]$, with
%probability $1-\beta$: if $a_i = \bot$ then $|\cE_S[\phi_i] - g_i|\leq T +
%\tau$ and if $a_i \in \mathbb{R}, |\cE_S[\phi_i] - a_i|\leq \tau$, provided
%that $n\ge n_0(\tau,\beta,\epsilon)$ and that $|\{i : |\cE_S[\phi_i]-g_i|
%\geq T - \tau\}| \leq r$. (If this last condition does not hold, the
%algorithm may halt early and stop accepting queries)
%\end{theorem}

We observe that the \emph{na\"ive} method of answering queries using their
empirical average allows us to answer each query up to accuracy $\tau$ with
probability $1-\beta$ given a data set of size $n_0 \geq \ln(2/\beta)/\tau^2$
so long as the queries are non-adaptively chosen. Thus, with high
probability, problems only arise between rounds of adaptivity. If we knew
when these rounds of adaptivity occurred, we could refresh our sample between
each round, and obtain total sample complexity linear in the number of rounds
of adaptivity. The method we present (using $(\epsilon,0)$-differential
privacy) lets us get a comparable bound \emph{without} knowing where the
rounds of adaptivity appear. Using $(\epsilon,\delta)$ privacy would allow us
to obtain constant factor improvements if the number of queries was large
enough, but does not get an asymptotically better dependence on the number of
rounds $r$ (it would allow us to reuse the round testing set \emph{quadratically}
many times, but we would still potentially need to refresh the training set
after each round of adaptivity, in the worst case).

The idea is the following: we obtain $r$ different estimation samples
$S_1,\ldots,S_r$ each of size sufficient to answer non-adaptively chosen
queries to error $\tau/8$ with probability $1-\beta/3$, and a separate round
detection sample $S_h$ of size $n_{SV}(\tau/8,\beta/3,\epsilon)$ for
$\epsilon = \tau/16$, which we access only through a copy of \SPARSE we
initialize with threshold $T = \tau/4$. As queries $\phi_i$ start arriving,
we compute their answers $a_i^t = \cE_{S_1}[\phi_i]$ using the na\"ive method
on estimation sample $S_1$ which we use as our \emph{guess} of the correct
value on $S_h$ when we feed $\phi_i$ to \SPARSE. If the answer \SPARSE
returns is $a^h_i = \bot$, then we know that with probability $1-\beta/3$,
$a_i^t$ is accurate up to tolerance $T + \tau/8 = 3\tau/8$ with respect to
$S_h$, and hence statistically valid up to tolerance $\tau/2$ by Theorem
\ref{thm:simulate-direct} with probability at least $1-2\beta/3$. Otherwise,
we discard our estimation set $S_1$ and continue with estimation set $S_2$.
We know that with probability $1-\beta/3$, $a_i^h$ is accurate with respect
to $S_h$ up to tolerance $\tau/8$, and hence statistically valid up to
tolerance $\tau/4$ by Theorem \ref{thm:simulate-direct} with probability at
least $1-2\beta/3$. We continue in this way, discarding and incrementing our
estimation set whenever our guess $g_i$ is incorrect. This succeeds in
answering every query so long as our guesses are not incorrect more than $r$
times in total. Finally, we know that except with probability at most
$m\beta/3$, by the accuracy guarantee of our estimation set for
\emph{non-adaptively} chosen queries, the only queries $i$ for which our
guesses $g_i$ will deviate from $\cE_{S_h}[\phi_i]$ by more than $T - \tau/8
= \tau/8$ are those queries that lie \emph{between rounds of adaptivity}.
There are at most $r$ of these by assumption, so the algorithm runs to
completion with probability at least $1 - m\beta/3$. The algorithm is given
in figure \ref{fig:reuse}.

\def\roundalgo{EffectiveRounds}
\begin{figure}[h]
\begin{boxedminipage}{\textwidth}
\textbf{Algorithm} \roundalgo

\textbf{Input:} A database $\bS$ of size $|\bS| \geq \frac{1156
r\ln(\frac{12}{\beta})}{\tau^2}$.

\textbf{Initialization:} Randomly split $\bS$ into $r+1$ sets: $r$ sets
$S_1,\ldots,S_r$ with size $|S_i| \geq
\frac{4\ln(\frac{12}{\beta})}{\tau^2}$, and one set $S_h$ with size $|S_h| =
\frac{1152\cdot r\cdot \ln(\frac{12}{\beta})}{\tau^2}$. Instantiate \SPARSE
with input $S_h$ and parameters $T = \tau/4$, $\tau' = \tau/8$, $\beta' =
\beta/3$, and $\epsilon = \tau/16$. Let $c \leftarrow 1$.

\textbf{Query stage}
 For each query $\phi_i$ do:
\begin{enumerate}
\item Compute $a_i^t =\cE_{S_c}[\phi_i]$. Let $g_i = a_i^t$ and feed
    $(\phi_i, g_i)$ to sparse and receive answer $a_i^h$.
\item If $a_i^h = \bot$ then return answer $a_i = a_i^t$.
\item Else return answer $a_i = a_i^h$. Set $c \leftarrow c + 1$.
\item If $c > r$ HALT.
\end{enumerate}
\end{boxedminipage}
\caption{The \roundalgo\ algorithm} \label{fig:reuse}
\end{figure}

This algorithm yields the following theorem:

\begin{theorem}
\label{cor:reusable} Let $\tau,\beta>0$ and define
\[
n_{SV}(\tau,\beta) = \frac{r\ln(\frac{1}{\beta})}{\tau^2}\,.
\]
There is an algorithm which on input of a data set $S$ of size $n$ sampled
from $\cP^n$ accepts any sequence of $m$ adaptively chosen queries
$\phi_1,\dots,\phi_m$ generated with at most $r$ rounds of adaptivity. With
probability at least $1 - m\beta$ the algorithm runs to completion and
returns estimates $a_1,\dots,a_m$ for each query. These estimates have the
property that for all $i\in[m]$ we have $\Pr\left[\left|\cP[\phi_i]-
a_i\right|>\tau\right]\le\beta$ provided that $n\ge C n_{SV}(\tau,\beta)$ for
sufficiently large constant $C$.
\end{theorem}

\begin{remark}
Note that the accuracy guarantee of \SPARSE depends only on the number of
incorrect guesses that are actually made. Hence, EffectiveRounds does not
halt until the actual number of instances of over-fitting to the estimation
samples $S_i$ is larger than $r$. This could be equal to the number of rounds
of adaptivity in the worst case (for example, if the analyst is running the
Dinur-Nissim reconstruction attack within each round \cite{DN03}), but in
practice might achieve a much better bound (if the analyst is not fully
adversarial).
\end{remark}

\remove{
In the full version of this work we show how the statements in Section \ref{sec:our-results} follow easily from the known differentially private algorithms, namely, the Multiplicative Weights Updates technique, the basic Laplacian noise addition and the Sparse Vector technique \cite{DworkFHPRR14:arxiv}.
}

\paragraph*{Acknowledgements}
We would like to thank Sanjeev Arora, Nina
Balcan, Avrim Blum, Dean Foster, Michael Kearns, Jon Kleinberg, Sasha Rakhlin,
and Jon Ullman for enlightening discussions and helpful comments. We also thank the Simons Institute for
Theoretical Computer Science at Berkeley where part of this research was done.

\ifsicomp
\bibliographystyle{siam}
\else
\bibliographystyle{alpha}
\fi

\bibliography{../refs}

\newcommand{\etalchar}[1]{$^{#1}$}
\begin{thebibliography}{DFH{\etalchar{+}}15b}

\bibitem[ANR11]{AharoniNR11}
Ehud Aharoni, Hani Neuvirth, and Saharon Rosset.
\newblock The quality preserving database: A computational framework for
  encouraging collaboration, enhancing power and controlling false discovery.
\newblock {\em IEEE/ACM Trans. Comput. Biology Bioinform.}, 8(5):1431--1437,
  2011.

\bibitem[AR14]{AharoniRosset13}
Ehud Aharoni and Saharon Rosset.
\newblock Generalized a-investing: definitions, optimality results and
  application to public databases.
\newblock {\em Journal of the Royal Statistical Society: Series B (Statistical
  Methodology)}, 76(4):771--794, 2014.

\bibitem[BDMN05]{BlumDMN05}
Avrim Blum, Cynthia Dwork, Frank McSherry, and Kobbi Nissim.
\newblock Practical privacy: the {SuLQ} framework.
\newblock In {\em PODS}, pages 128--138, 2005.

\bibitem[BE02]{BousquettE02}
Olivier Bousquet and Andr{\'e} Elisseeff.
\newblock Stability and generalization.
\newblock {\em JMLR}, 2:499--526, 2002.

\bibitem[BE12]{BegleyEllis12}
C.~Glenn Begley and Lee Ellis.
\newblock Drug development: {Raise} standards for preclinical cancer research.
\newblock {\em Nature}, 483:531--533, 2012.

\bibitem[BH95]{benjaminihochberg95}
Yoav Benjamini and Yosef Hochberg.
\newblock Controlling the false discovery rate -- a practical and powerful
  approach to multiple testing.
\newblock {\em Journal of the Royal Statistics Society: Series B (Statistical
  Methodology)}, 57:289--300, 1995.

\bibitem[BH15]{BlumH15}
Avrim Blum and Moritz Hardt.
\newblock The ladder: {A} reliable leaderboard for machine learning
  competitions.
\newblock {\em CoRR}, abs/1502.04585, 2015.

\bibitem[BNS{\etalchar{+}}15]{BassilyNSSSU15}
Raef Bassily, Kobbi Nissim, Adam~D. Smith, Thomas Steinke, Uri Stemmer, and
  Jonathan Ullman.
\newblock Algorithmic stability for adaptive data analysis.
\newblock {\em CoRR}, abs/1511.02513, 2015.

\bibitem[BST14]{BassilyST14}
Raef Bassily, Adam Smith, and Abhradeep Thakurta.
\newblock Private empirical risk minimization, revisited.
\newblock {\em CoRR}, abs/1405.7085, 2014.

\bibitem[CKL{\etalchar{+}}06]{ChuKLYBNO:06}
C.~Chu, S.~Kim, Y.~Lin, Y.~Yu, G.~Bradski, A.~Ng, and K.~Olukotun.
\newblock Map-reduce for machine learning on multicore.
\newblock In {\em Proceedings of NIPS}, pages 281--288, 2006.

\bibitem[DFH{\etalchar{+}}15a]{DworkFHPRR15:arxiv}
Cynthia Dwork, Vitaly Feldman, Moritz Hardt, Toniann Pitassi, Omer Reingold,
  and Aaron Roth.
\newblock Generalization in adaptive data analysis and holdout reuse.
\newblock {\em CoRR}, abs/1506, 2015.

\bibitem[DFH{\etalchar{+}}15b]{DworkFHPRR15:science}
Cynthia Dwork, Vitaly Feldman, Moritz Hardt, Toniann Pitassi, Omer Reingold,
  and Aaron Roth.
\newblock The reusable holdout: Preserving validity in adaptive data analysis.
\newblock {\em Science}, 349(6248):636--638, 2015.

\bibitem[DKM{\etalchar{+}}06]{DKMMN06}
Cynthia Dwork, Krishnaram Kenthapadi, Frank McSherry, Ilya Mironov, and Moni
  Naor.
\newblock Our data, ourselves: Privacy via distributed noise generation.
\newblock In {\em EUROCRYPT}, pages 486--503, 2006.

\bibitem[DMNS06]{DMNS06}
Cynthia Dwork, Frank McSherry, Kobbi Nissim, and Adam Smith.
\newblock Calibrating noise to sensitivity in private data analysis.
\newblock In {\em Theory of Cryptography}, pages 265--284. Springer, 2006.

\bibitem[DN03]{DN03}
Irit Dinur and Kobbi Nissim.
\newblock Revealing information while preserving privacy.
\newblock In {\em PODS}, pages 202--210. ACM, 2003.

\bibitem[DN04]{DworkNi04}
Cynthia Dwork and Kobbi Nissim.
\newblock Privacy-preserving datamining on vertically partitioned databases.
\newblock In {\em CRYPTO}, pages 528--544, 2004.

\bibitem[DR14]{DR14}
Cynthia Dwork and Aaron Roth.
\newblock The algorithmic foundations of differential privacy.
\newblock {\em Foundations and Trends in Theoretical Computer Science},
  9(34):211--407, 2014.

\bibitem[Dwo11]{Dwork11}
Cynthia Dwork.
\newblock A firm foundation for private data analysis.
\newblock {\em CACM}, 54(1):86--95, 2011.

\bibitem[FGR{\etalchar{+}}13]{FeldmanGRVX:13}
Vitaly Feldman, Elena Grigorescu, Lev Reyzin, Santosh Vempala, and Ying Xiao.
\newblock Statistical algorithms and a lower bound for planted clique.
\newblock In {\em STOC}, pages 655--664. ACM, 2013.

\bibitem[Fre83]{Freedman83}
David~A. Freedman.
\newblock A note on screening regression equations.
\newblock {\em The American Statistician}, 37(2):152--155, 1983.

\bibitem[FS08]{FosterStine08}
D.~Foster and R.~Stine.
\newblock Alpha-investing: A procedure for sequential control of expected false
  discoveries.
\newblock {\em J. Royal Statistical Soc.: Series B (Statistical Methodology)},
  70(2):429--444, 2008.

\bibitem[GL14]{GelmanLoken13}
Andrew Gelman and Eric Loken.
\newblock The statistical crisis in science.
\newblock {\em The American Statistician}, 102(6):460, 2014.

\bibitem[HR10]{HardtR10}
Moritz Hardt and Guy~N. Rothblum.
\newblock A multiplicative weights mechanism for privacy-preserving data
  analysis.
\newblock In {\em 51st IEEE FOCS 2010}, pages 61--70, 2010.

\bibitem[HTF09]{hastie2009elements}
Trevor Hastie, Robert Tibshirani, and Jerome~H. Friedman.
\newblock {\em The Elements of Statistical Learning: Data Mining, Inference,
  and Prediction}.
\newblock Springer series in statistics. Springer, 2009.

\bibitem[HU14]{HU14}
Moritz Hardt and Jonathan Ullman.
\newblock Preventing false discovery in interactive data analysis is hard.
\newblock In {\em FOCS}, pages 454--463, 2014.

\bibitem[Ioa05a]{Ioannidis05b}
John~A. Ioannidis.
\newblock Contradicted and initially stronger effects in highly cited clinical
  research.
\newblock {\em The Journal of American Medical Association}, 294(2):218--228,
  2005.

\bibitem[Ioa05b]{Ioannidis05}
John P.~A. Ioannidis.
\newblock {Why Most Published Research Findings Are False}.
\newblock {\em PLoS Medicine}, 2(8):124, August 2005.

\bibitem[Kaga]{KaggleRolie}
Five lessons from {Kaggle's} event recommendation engine challenge.
\newblock
  \url{http://www.rouli.net/2013/02/five-lessons-from-kaggles-event.html}.
\newblock Accessed: 2014-10-07.

\bibitem[Kagb]{KaggleBlog}
Kaggle blog: No free hunch.
\newblock \url{http://blog.kaggle.com/}.
\newblock Accessed: 2014-10-07.

\bibitem[Kagc]{KaggleForums}
Kaggle user forums.
\newblock \url{https://www.kaggle.com/forums}.
\newblock Accessed: 2014-10-07.

\bibitem[Kea98]{Kearns93}
Michael Kearns.
\newblock Efficient noise-tolerant learning from statistical queries.
\newblock {\em Journal of the ACM (JACM)}, 45(6):983--1006, 1998.

\bibitem[KV94]{KV94}
Michael~J Kearns and Umesh~Virkumar Vazirani.
\newblock {\em An introduction to computational learning theory}.
\newblock MIT press, 1994.

\bibitem[MNPR06]{MukherjeeNPR06}
Sayan Mukherjee, Partha Niyogi, Tomaso Poggio, and Ryan Rifkin.
\newblock Learning theory: stability is sufficient for generalization and
  necessary and sufficient for consistency of empirical risk minimization.
\newblock {\em Advances in Computational Mathematics}, 25(1-3):161--193, 2006.

\bibitem[PRMN04]{PoggioRMN04}
Tomaso Poggio, Ryan Rifkin, Sayan Mukherjee, and Partha Niyogi.
\newblock General conditions for predictivity in learning theory.
\newblock {\em Nature}, 428(6981):419--422, 2004.

\bibitem[PSA11]{PrintzSA11}
Florian Prinz, Thomas Schlange, and Khusru Asadullah.
\newblock Believe it or not: how much can we rely on published data on
  potential drug targets?
\newblock {\em Nature Reviews Drug Discovery}, 10(9):712--712, 2011.

\bibitem[Reu03]{Reunanen:03}
Juha Reunanen.
\newblock Overfitting in making comparisons between variable selection methods.
\newblock {\em Journal of Machine Learning Research}, 3:1371--1382, 2003.

\bibitem[RF08]{RaoF:08}
R.~Bharat Rao and Glenn Fung.
\newblock On the dangers of cross-validation. an experimental evaluation.
\newblock In {\em International Conference on Data Mining}, pages 588--596.
  SIAM, 2008.

\bibitem[RR10]{RR10}
Aaron Roth and Tim Roughgarden.
\newblock Interactive privacy via the median mechanism.
\newblock In {\em $42$nd ACM STOC}, pages 765--774. ACM, 2010.

\bibitem[RZ15]{RussoZ15}
Daniel Russo and James Zou.
\newblock Controlling bias in adaptive data analysis using information theory.
\newblock {\em CoRR}, abs/1511.05219, 2015.

\bibitem[SNS11]{SimmonsNS11}
Joseph~P. Simmons, Leif~D. Nelson, and Uri Simonsohn.
\newblock False-positive psychology: Undisclosed flexibility in data collection
  and analysis allows presenting anything as significant.
\newblock {\em Psychological Science}, 22(11):1359--1366, 2011.

\bibitem[SSBD14]{Shalev-ShwartzBen-David:2014}
Shai Shalev-Shwartz and Shai Ben-David.
\newblock {\em Understanding Machine Learning: From Theory to Algorithms}.
\newblock Cambridge University Press, 2014.

\bibitem[SSSSS10]{ShwartzSSS10}
Shai Shalev-Shwartz, Ohad Shamir, Nathan Srebro, and Karthik Sridharan.
\newblock Learnability, stability and uniform convergence.
\newblock {\em The Journal of Machine Learning Research}, 11:2635--2670, 2010.

\bibitem[SU14]{SU14}
Thomas Steinke and Jonathan Ullman.
\newblock Interactive fingerprinting codes and the hardness of preventing false
  discovery.
\newblock {\em arXiv preprint arXiv:1410.1228}, 2014.

\bibitem[TT15]{TaylorT15}
Jonathan Taylor and Robert~J. Tibshirani.
\newblock Statistical learning and selective inference.
\newblock {\em Proceedings of the National Academy of Sciences},
  112(25):7629--7634, 2015.

\bibitem[Win]{KaggleBlogWind}
David Wind.
\newblock Learning from the best.
\newblock \url{http://blog.kaggle.com/2014/08/01/learning-from-the-best/}.
\newblock Accessed: 2014-10-07.

\bibitem[WLF15]{WangLF15}
Yu{-}Xiang Wang, Jing Lei, and Stephen~E. Fienberg.
\newblock Learning with differential privacy: Stability, learnability and the
  sufficiency and necessity of {ERM} principle.
\newblock {\em CoRR}, abs/1502.06309, 2015.

\end{thebibliography}

\appendix
\section{Adaptivity in fitting a linear model}
\label{app:overfit}

In this section, we give a very simple example to illustrate how a data
analyst could end up overfitting to a dataset by asking only a small number of
(adaptively) chosen queries to the dataset, if they are answered using the
na\"ive method.

 The data analyst has $n$ samples $D=\{x_1,\dots,x_n\}$ over
$d$ real-valued attributes sampled from an unknown distribution~${\cal D}$.
The analyst's goal is to find a linear model~$\ell$ that maximizes the
average correlation with the unknown distribution. Formally, the goal is to
find a unit vector that maximizes the function
\[
f(u) = \E_{x\sim{\cal D}}[\langle u,x\rangle]\,.
\]
Not knowing the distribution the analyst decides to solve the corresponding
optimization problem on her finite sample:
\[
\tilde f_D(u) = \frac1n\sum_{x\in D} \langle u,x\rangle\,.
\]
The analyst attempts to solve the problem using the following simple but \emph{adaptive
strategy:}
\begin{enumerate}
\item For $i=1,\dots,d,$ determine  $s_i = \mathrm{sign}\Big(\sum_{x\in D} x_i\Big).$
\item Let $\tilde u = \frac1{\sqrt{d}} (s_1,\dots,s_d).$
\end{enumerate}
Intuitively, this natural approach first determines for each attribute whether it is
positively or negatively correlated. It then aggregates this information
across all $d$ attributes into a single linear model.

The next lemma shows that this adaptive strategy has a terrible generalization
performance (if $d$ is large). Specifically, we show that even if there is no
linear structure whatsoever in the underlying distribution (namely it is
normally distributed), the analyst's strategy falsely discovers a linear model
with large objective value.
\begin{lemma}
Suppose ${\cal D}=N(0,1)^d.$ Then,
every unit vector $u\in\mathbb{R}^d$
satisfies $f(u) = 0.$ However,
$\E_D [\tilde f_D(\tilde u)] = \sqrt{2/\pi}\cdot \sqrt{d/n}.$
%In particular,
%\[
%\frac{\E_D\tilde f_D(\tilde u)}{f(\tilde u)}= \sqrt{\frac dn}\,.
%\]
\end{lemma}
\begin{proof}
The first claim follows because $\langle u,x\rangle$ for $x\sim N(0,1)^d$ is
distributed like a Gaussian random variable $N(0,1).$
Let us now analyze the objective value of $\tilde u.$
\begin{align*}
\tilde f_D(\tilde u)
 = \frac1n\sum_{x\in D} \frac{s_i}{\sqrt{d}}\sum_{i=1}^d x_i
 = \frac1{\sqrt{d}}\sum_{i=1}^d \left|\frac1n\sum_{x\in D}x_i \right|
\end{align*}
Hence,
\[
\E_D [\tilde f_D(\tilde u)]
 = \sum_{i=1}^d \frac{1}{\sqrt{d}}\E_D\lb\left|\frac1n\sum_{x\in D}x_i \right|\rb.
\]
Now, $(1/n)\sum_{x\in D}x_i$ is distributed like a gaussian random variable
$g\sim N(0,1/n),$ since each $x_i$ is a standard gaussian. It follows that
\[
\E_D \tilde f_D(\tilde u) = \sqrt{\frac{2d }{\pi n}}.
\]
\end{proof}

Note that all the operations performed by the analyst are based on empirical averages of real-valued functions. To determine the bias, the function is just $\phi_i(x)= x_i$ and to determine the final correlation it is $\psi(x) = \langle u,x\rangle$. These functions are not bounded to the range $[0,1]$ as required by the formal definition of our model. However, it is easy to see that this is a minor issue. Note that both $x_i$ and $\langle u,x\rangle$ are distributed according to $N(0,1)$ whenever $x \sim N(0,1)^d$. This implies that for every query function $\phi$ we used, $\pr[|\phi(x)| \geq B] \leq 1/\poly(n,d)$ for some $B=O(\log (dn))$. We can therefore truncate and rescale each query as $\phi'(x) = P_B(\phi(x))/(2B)+1/2$, where $P_B$ is the truncation of the values outside $[-B,B]$. This ensures that the range of $\phi'(x)$ is $[0,1]$. It is easy to verify that using these $[0,1]$-valued queries does not affect the analysis in any significant way (aside from scaling by a logarithmic factor) and we obtain overfitting in the same way as before (for large enough $d$).
\remove{
Remarks (mainly for ourselves):
\begin{enumerate}
\item
We remark that $\tilde u$ is not actually the optimizer of the $\tilde
f_D(u).$ The optimizer is given by $v\in\R^d$ so that (up to scaling) $v_i =
\sum_{x\in D}x_i.$ A similar analysis applies to this vector and gives a
similar quantitative trade-off. In particular, if we use convex optimization
to solve the problem, we end up with this vector. For convex optimization it
would make sense to talk about minimizing a convex function. The natural
objective function would be $\|u\|^2 - \sum_{x\in D}\langle x,u\rangle$ and it
has this optimizer which we can see by taking gradients.
\item
The example extends to a classification setting, but perhaps this setting is
more natural from a hypothesis testing perspective.
\item
Here's a general approach for coming up with good examples of what goes wrong
with adaptivity when $d>n.$ Thinking of the data set as a $n\times d$
matrix~$A$, we know there is a nonzero vector in the kernel of~$A.$ This means
that we can find a vector whose inner product with every sample vector is $0$
but its correlation with the distribution can be anything. (My work with David
Woodruff shows that such a vector can always be found adaptively given
accurate answers to linear queries.)
\end{enumerate}
}

\section{Background on Differential Privacy}
\label{sec:background}

When applying $(\epsilon,\delta)$-differential privacy,
we are typically interested in values of $\delta$ that are very small compared
to~$n$. In particular, values of $\delta$ on the order of $1/n$ yield no
meaningful definition of privacy as they permit the publication of the
complete records of a small number of data set participants---a violation of
any reasonable notion of privacy.

\begin{theorem}
\label{thm:group-privacy}
Any $(\eps,\delta)$-differentially private mechanism~$\cm$
satisfies for all pairs of data sets $\dba,\dbb$ differing in at most $k$ elements, and all $\cs \subseteq \mathrm{Range}(\cm)$:
\[
\Pr[\cm(\dba) \in \cs] \leq \exp(k\eps)\Pr[\cm(\dbb)\in \cs]
+e^{\epsilon(k-1)}\delta,
\]
where the probability space is over the coin flips of the mechanism $\cm$.
\end{theorem}

Differential privacy also degrades gracefully under composition.  It is easy to see that the independent use of an $(\eps_1,0)$-differentially private algorithm and an $(\eps_2,0)$-differentially private algorithm, when taken together, is $(\eps_1+\eps_2,0)$-differentially private.
More generally, we have

\begin{theorem}
\label{thm:easy-composition}
Let $\cm_i:\rowdom\rightarrow \mathcal{R}_i$ be an $(\eps_i,\delta_i)$-differentially private algorithm for $i \in [k]$. Then if $\cm_{[k]} :\rowdom\rightarrow \prod_{i=1}^k\mathcal{R}_i$ is defined to be $\cm_{[k]}(\dba) = (\cm_1(\dba),\ldots,\cm_k(\dba))$, then $\cm_{[k]}$ is $(\sum_{i=1}^k\eps_i,\sum_{i=1}^k \delta_i)$-differentially private.
\end{theorem}

A more sophisticated argument yields significant improvement when $\eps < 1$:
\begin{theorem}
\label{thm:composition-advanced}
For all $\eps,\delta, \delta' \geq 0$, the composition of $k$ arbitrary $(\eps,\delta)$-differentially private mechanisms is $(\eps',k\delta+\delta')$-differentially private, where
$$\eps' = \sqrt{2k\ln(1/\delta')}\eps + k\eps(e^\eps-1),$$
even when the mechanisms are chosen adaptively.
\end{theorem}
\noindent

Theorems~\ref{thm:easy-composition} and~\ref{thm:composition-advanced} are very general. For example, they apply to queries posed to overlapping, but not identical, data sets.  Nonetheless, data utility will eventually be consumed: the Fundamental Law of Information Recovery states that overly accurate answers to too many questions will destroy privacy in a spectacular way (see \cite{DN03} {\it et sequelae}).  The goal of algorithmic research on differential privacy is to stretch a given privacy ``budget'' of, say, $\eps_0$, to provide as much utility as possible, for example, to provide useful answers to a great many counting queries.  The bounds afforded by the composition theorems are the first, not the last, word on utility.

\section{Concentration and moment bounds}
\subsection{Concentration inequalities}
We will use the following statement of the multiplicative Chernoff bound:
\begin{lemma}[Chernoff's bound]
\label{lem:bennett}
Let $Y_1,Y_2, \ldots, Y_n$ be i.i.d. Bernoulli random variables with expectation $p>0$.
Then for every $\gamma > 0$,
 $$\pr\lb \sum_{i\in [n]} Y_i \geq (1+\gamma) n p \rb \leq \exp\lp -np ( (1+\gamma) \ln(1+\gamma)- \gamma)\rp.$$
\end{lemma}

\begin{lemma}[McDiarmid's inequality]
\label{lem:mcdiarmid}
Let $X_1,X_2, \ldots, X_n$ be independent random variables taking values in the set $\X$. Further let $f:\X^n \rar \R$ be a function that satisfies, for all $i\in
[n]$ and $x_1,x_2,\ldots,x_n,x_i'\in\X$,
$$f(x_1,\ldots,x_i,\ldots,x_n) - f(x_1,\ldots,x_i',\ldots,x_n) \leq c.$$
Then for all $\al > 0$, and $\mu = \E\lb f(X_1,\ldots,X_n) \rb$, $$\pr\lb f(X_1,\ldots,X_n)- \mu \geq  \alpha \rb \leq \exp\lp\frac{-2\alpha^2}{n \cdot c^2}\rp.$$
\end{lemma}

\subsection{Moment Bounds}
\begin{lemma}
\label{lem:moment-bound-by-bernoulli}
Let $Y_1,Y_2, \ldots, Y_n$ be i.i.d. Bernoulli random variables with expectation $p$. We denote by $\mathcal{M}_k[B(n,p)] \doteq \E\lb \lp\fr{n}\sum_{i\in [n]} Y_i\rp^k \rb$.
Let $X_1,X_2, \ldots, X_n$ be i.i.d. random variables with values in $[0,1]$ and expectation $p$.
Then for every $k > 0$,
$$\E\lb \lp\fr{n}\sum_{i\in [n]} X_i\rp^k \rb \leq \mathcal{M}_k[B(n,p)].$$
\end{lemma}
\begin{proof}
We use $\bi$ to denote a $k$-tuple of indices $(i_1,\ldots,i_k) \in [n]^k$ (not necessarily distinct). For $I$ like that we denote by $\{\ell_1,\ldots,\ell_{k'}\}$ the set of distinct indices in $I$ and let $k_1,\ldots,k_{k'}$ denote their multiplicities. Note that $\sum_{j \in [k']} k_j = k$.
We first observe that
\alequ{\E\lb \lp\fr{n}\sum_{i\in [n]} X_i\rp^k \rb = \E_{I \sim [n]^k}\lb \E \lb \prod_{j\in [k]} X_{i_j} \rb \rb = \E_{I \sim [n]^k}\lb \E \lb \prod_{j\in [k']} X_{\ell_j}^{k_j} \rb \rb = \E_{I \sim [n]^k}\lb \prod_{j\in [k']} \E \lb X_{\ell_j}^{k_j} \rb \rb \label{eq:decompose-moment-2},
} where the last equality follows from independence of $X_i$'s. For every $j$, the range of $X_{\ell_j}$ is $[0,1]$ and thus
$$\E \lb X_{\ell_j}^{k_j} \rb \leq  \E \lb X_{\ell_j} \rb = p .$$ Moreover the value $p$ is achieved when $X_{\ell_j}$ is Bernoulli with expectation $p$. That is $$\E \lb X_{\ell_j}^{k_j} \rb \leq \E\lb Y_{\ell_j}^{k_j} \rb ,$$ and by using this in equality (\ref{eq:decompose-moment-2}) we obtain that
\alequn{\E\lb \lp\fr{n}\sum_{i\in [n]} X_i\rp^k \rb \leq \E\lb \lp\fr{n}\sum_{i\in [n]} Y_i\rp^k \rb = \mathcal{M}_k[B(n,p)].}
\end{proof}

\begin{lemma}
\label{lem:moment-bound}
For all integers $n\geq k \geq 1$ and $p \in [0,1]$,
$$\mathcal{M}_k[B(n,p)] \leq p^k + (k \ln n+1) \cdot \lp\frac{k}{n}\rp^k.$$
\end{lemma}
\begin{proof}
Let $U$ denote $\fr{n}\sum_{i\in [n]} X_i$, where $X_i$'s are i.i.d. Bernoulli random variables with expectation $p >0$ (the claim is obviously true if $p=0$). Then
\equ{\E[U^k] \leq p^k + \int_{p^k}^{1} \pr[U^k \geq t] dt \label{eq:bound-moment-base}.}
We substitute $t = (1+\gamma)^k p^k$ and observe that Lemma \ref{lem:bennett} gives:
$$\pr[U^k \geq t] = \pr[U^k \geq ((1+\gamma)p)^k] = \pr[U \geq (1+\gamma)p] \leq \exp\lp -np ( (1+\gamma) \ln(1+\gamma)- \gamma)\rp .$$
Using this substitution in eq.(\ref{eq:bound-moment-base}) together with $\frac{dt}{d\gamma} = k(1+\gamma)^{k-1} \cdot p^k$ we obtain
\alequ{\E[U^k] &\leq p^k + \int_0^{1/p-1} \exp\lp -np ( (1+\gamma) \ln(1+\gamma)- \gamma)\rp \cdot k(1+\gamma)^{k-1} d\gamma \nonumber
\\ & = p^k + p^k k \int_0^{1/p-1} \fr{1+\gamma} \cdot \exp\lp k\ln(1+\gamma) -np ( (1+\gamma) \ln(1+\gamma)- \gamma)\rp d\gamma  \nonumber
\\ & \leq p^k + p^k k \max_{\gamma \in [0,1/p-1]}\left\{ \exp\lp k\ln(1+\gamma) -np ( (1+\gamma) \ln(1+\gamma)- \gamma)\rp \right\} \cdot \int_0^{1/p-1} \fr{1+\gamma}  d\gamma \nonumber
\\ & = p^k + p^k k \ln(1/p) \cdot \max_{\gamma \in [0,1/p-1]}\left\{ \exp\lp k\ln(1+\gamma) -np ( (1+\gamma) \ln(1+\gamma)- \gamma)\rp \right\} \label{eq:exp-main}.}
We now find the maximum of $g(\gamma) \doteq k\ln(1+\gamma) -np ( (1+\gamma) \ln(1+\gamma)- \gamma)$.
Differentiating the expression we get $\frac{k}{1+\gamma} - np \ln(1+\gamma)$ and therefore the function attains its maximum at the (single) point $\gamma_0$ which satisfies: $(1+\gamma_0) \ln(1+\gamma_0) = \frac{k}{np}$. This implies that $\ln(1+\gamma_0) \leq \ln\left(\frac{k}{np}\right).$ Now we observe that  $(1+\gamma) \ln(1+\gamma)- \gamma$ is always non-negative and therefore $g(\gamma_0) \leq k\ln\left(\frac{k}{np}\right)$. Substituting this into eq.(\ref{eq:exp-main}) we conclude that
$$\E[U^k] \leq  p^k + p^k k \ln(1/p) \cdot \exp\lp k\ln\left(\frac{k}{np}\right)\rp = p^k + k \ln(1/p) \cdot \lp\frac{k}{n}\rp^k .$$
Finally, we observe that if $p\geq 1/n$ then clearly $\ln(1/p) \leq \ln n$ and the claim holds. For any $p < 1/n$ we use monotonicity of $\mathcal{M}_k[B(n,p)]$ in $p$ and upper bound the probability by the bound for $p=1/n$ that equals
$$\lp\frac{1}{n}\rp^k + (k \ln n) \cdot \lp\frac{k}{n}\rp^k \leq (k \ln n+1) \cdot \lp\frac{k}{n}\rp^k .$$
\end{proof}

\eat{

\begin{lemma}
\label{lem:chernoff-bound-by-moment}
Let $n,  \eps>0, p\geq 0$ and let $V$ be a non-negative random variable that for every $k\geq 0$, satisfies $\E[V^k] \leq e^{\eps k} \mathcal{M}_k[B(n,p)]$.
Then for any $\tau > 0$,
$$\pr[V \geq p+\tau] \leq e^{-2\tau'^2 n},$$ where $\tau' = e^{-\eps}(p+\tau)-p = e^{-\eps}\tau - (1-e^{-\eps})p.$
In particular, for $\tau \leq 1/4$ and $\eps=\tau/2$ we have $\tau' \geq \tau/3$.
\end{lemma}
\begin{proof}
Consider the moment generating function of $V$, that is $g(t) = \E[e^{tV}]$.
By definition, for every $t \geq 0$, $$\E[e^{tV}] = \sum_{i=0}^\infty \frac{t^i}{i!}\E[V^i] \leq \sum_{i=0}^\infty \frac{(e^{\eps}t)^i}{i!}\E[B^k(n,p)]=h(e^{\eps}t),$$
where $h(t) = (1-p + pe^{t})^n$ is the moment-generating function of the binomial distribution $B(n,p)$.

Now, following the standard proof of the Chernoff bound, for every $t > 0$,
\alequn{\pr[V \geq p+\tau] = \pr[e^{tV} \geq e^{p+\tau}] \leq \frac{\E[e^{tV}]}{e^{t(p+\eps)}}  = \frac{(1-p + pe^{e^\eps t})^n}{e^{t(p+\eps)}}}
The expression above is minimized for $t_0 = \ln\lp\frac{(1-p)q}{(1-q)p}\rp$, where $q= e^{-\eps}(p+\tau)$. Further,
$$\frac{(1-p + pe^{e^\eps t_0})^n}{e^{t_0(p+\eps)}} \leq e^{-2(q-p)^2n}. $$ Hence
$$\pr[V \geq p+\tau] \leq e^{-2\tau'^2 n},$$ where $\tau' = q-p = e^{-\eps}(p+\tau)-p = e^{-\eps}\tau - (1-e^{-\eps})p$.
Note that if $\eps=\tau/2$ and $\tau \leq 1/4$ then $e^{-\eps} \geq 5/6$ and $(1 - e^{-\eps})p \leq \eps p \leq \tau/2$. Hence $\tau' \geq \tau/3$.
\end{proof}
}

\begin{lemma}
\label{lem:markov-bound-by-moment}
Let $n>k>0,  \eps>0, p>0,\delta \geq 0$ and let $V$ be a non-negative random variable that satisfies $\E[V^k] \leq e^{\eps k} \mathcal{M}_k[B(n,p)] + \delta$.
Then for any $\tau \in [0,1/3]$, $\beta \in (0,2/3]$ if
\begin{itemize}
\item $\eps \leq \tau/2$,
\item $k \geq \max\{4 p \ln(2/\beta)/\tau,\ 2\log\log n \}$,
\item $n \geq 3 k/\tau$   then
\end{itemize}
$$\pr[V \geq p+\tau] \leq \beta + \delta/(p+\tau)^k.$$
\end{lemma}
\begin{proof}
Observe that by Markov's inequality:
$$\pr[V \geq p+\tau] = \pr[V^k \geq (p+\tau)^k] \leq \frac{\E[V^k]}{(p+\tau)^k} \leq \frac{e^{\eps k} \mathcal{M}_k[B(n,p)] }{p^k(1+\tau/p)^k} +\frac{\delta}{(p+\tau)^k} .$$
Using Lemma \ref{lem:moment-bound} we obtain that
\equ{\pr[V \geq p+\tau] \leq \frac{p^k + (k \ln n+1) \cdot \lp\frac{k}{n}\rp^k}{e^{-\eps k} p^k (1+\tau/p)^k } + \frac{\delta}{(p+\tau)^k}= \frac{1 + (k \ln n+1) \cdot \lp\frac{k}{pn}\rp^k}{(e^{-\eps}(1+\tau/p))^k} + \frac{\delta}{(p+\tau)^k}.\label{eq:v-bound}}
Using the condition $\eps \leq \tau/2$ and $\tau \leq 1/3$ we first observe that
\equn{e^{-\eps}(1+\tau/p) \geq (1-\eps)(1+\tau/p) = 1 + \tau/p - \eps - \eps \tau/p \geq 1+\tau/(3p) .}
Hence, with the condition that $k \geq 4 p \ln (2/\beta)/\tau$ we get
\equ{(e^{-\eps}(1+\tau/p))^k \geq (1+\tau/(3p))^k \geq e^{k\tau/(4p)} \geq \frac{2}{\beta}. \label{eq:beta-1}}
Using the condition $n \geq 3 k/\tau$.
\equn{ e^{-\eps}\tau/p \geq 3e^{-\eps}k/(np) > 2k/(np).}
Together with the condition $k \geq \max\{4 \ln(2/\beta)/\tau,\ 2\log\log n \}$, we have
$$\log(2/\beta) + \log(k \ln n + 1) \leq \log(2/\beta) + \log(k+1) + \log\log n \leq k$$ since $k/2 \geq \log\log n$  holds by assumption and for $k \geq 12 \ln(2/\beta)$, $k/6 \geq \log(2/\beta)$ and $k/3 \geq \log(k+1)$ (whenever $\beta < 2/3$). Therefore we get
\equ{ (e^{-\eps}(1+\tau/p))^k \geq (e^{-\eps}\tau/p)^k \geq 2^k \cdot \lp\frac{k}{pn}\rp^k \geq \frac{2}{\beta} \cdot (k \ln n+1) \cdot \lp\frac{k}{pn}\rp^k. \label{eq:beta-2}}
Combining eq.(\ref{eq:beta-1}) and (\ref{eq:beta-2}) we obtain that
\equn{\frac{1 + (k \ln n+1) \cdot \lp\frac{k}{pn}\rp^k}{(e^{-\eps}(1+\tau/p))^k} \leq \beta/2 + \beta/2 = \beta.}
Substituting this into eq.(\ref{eq:v-bound}) we obtain the claim.
\end{proof}

\end{document}